\documentclass[journal]{IEEEtran}
\usepackage{multirow}
\usepackage{wrapfig}
\usepackage{float}
\usepackage{tabularx}
\usepackage{longtable}
\usepackage{color}
\definecolor{menucolor}{rgb}{0.1,0.52,0.47}
\definecolor{anchorcolor}{rgb}{0.85,0.37,0.01}
\definecolor{runcolor}{rgb}{0.46,0.44,0.701}
\definecolor{linkcolor}{rgb}{0.3,0.55,0.01}
\definecolor{urlcolor}{rgb}{0.12,0.47,0.70}
\definecolor{citecolor}{rgb}{0.55,0.36,0.01}
\definecolor{filecolor}{rgb}{0.4,0.4,0.4}
\usepackage[colorlinks=true, urlcolor=urlcolor, linkcolor=linkcolor,citecolor=citecolor,filecolor=filecolor, anchorcolor=anchorcolor,menucolor=menucolor]{hyperref}
\usepackage{graphicx}
\usepackage{verbatim}
\usepackage{booktabs}
\usepackage{tikz}
\ifCLASSOPTIONcompsoc
  \usepackage[caption=false,font=normalsize,labelfont=sf,textfont=sf]{subfig}
\else
\usepackage[caption=false,font=footnotesize]{subfig}
\fi
\usepackage{array}

\usepackage{multirow}
\usepackage{wrapfig}
\usepackage{tabularx}
\usepackage{longtable}
\usepackage{hyperref}
\usepackage{amsmath}
\usepackage{amsthm}
\usepackage{graphicx}
\graphicspath{{./temp/}}
\usepackage{verbatim}
\usepackage[nocompress]{cite}

\usepackage{xcolor}
\usepackage{tikz}

\definecolor{cbp1}{RGB}{166,206,227}
\definecolor{cbp2}{RGB}{31,120,180}
\definecolor{cbp3}{RGB}{253,191,111}
\definecolor{cbp4}{RGB}{255,127,0}
\definecolor{dfp1}{RGB}{27,158,119}
\definecolor{dfp2}{RGB}{166,206,227}
\definecolor{dfp6}{RGB}{217,95,2}
\definecolor{dfp4}{RGB}{230,171,2}
\definecolor{dfp5}{RGB}{117,112,179}
\definecolor{dfp3}{RGB}{231,41,138}
\definecolor{dfp7}{RGB}{102,166,30}
\definecolor{dfp8}{RGB}{166,118,29}
\definecolor{dfp9}{RGB}{102,102,102}
\definecolor{scp1}{RGB}{166,206,227}
\definecolor{scp3}{RGB}{178,223,138}
\definecolor{scp4}{RGB}{51,160,44}
\definecolor{scp5}{RGB}{251,154,153}
\definecolor{scp6}{RGB}{227,26,28}
\definecolor{scp11}{RGB}{255,255,153}
\definecolor{dkp1}{RGB}{173,216,230}
\definecolor{dkp2}{RGB}{139,0,0}

\definecolor{gtex1}{RGB}{166,206,227}
\definecolor{gtex2}{RGB}{31,120,180}

\definecolor{gtex3}{RGB}{105,105,105}
\definecolor{gtex4}{RGB}{51,160,44}
\definecolor{gtex5}{RGB}{251,154,153}
\definecolor{gtex6}{RGB}{227,26,28}
\definecolor{gtex7}{RGB}{106,61,154}
\definecolor{gtex8}{RGB}{255,255,0}
\definecolor{gtex9}{RGB}{202,178,214}
\definecolor{gtex10}{RGB}{106,61,154}

\definecolor{indic1}{RGB}{255,   255,    0}
\definecolor{indic2}{RGB}{31,   120,  180}
\definecolor{indic3}{RGB}{177,    89,  40}
\definecolor{indic4}{RGB}{51,   160,   44}
\definecolor{indic5}{RGB}{160,  32,  240}
\definecolor{indic6}{RGB}{227,   26,   28}
\definecolor{indic7}{RGB}{253,   191,   111}
\definecolor{indic12}{RGB}{177,   89,   40}

\definecolor{zip1}{RGB}{27,  158,  119}
\definecolor{zip2}{RGB}{255,   127,   0}
\definecolor{zip3}{RGB}{153,    50,  204}
\definecolor{zip4}{RGB}{255,   255,    0}
\definecolor{zip5}{RGB}{124,   252,    0}
\definecolor{zip6}{RGB}{24,   116, 205}
\definecolor{zip7}{RGB}{166,   118,   29}
\definecolor{zip8}{RGB}{46,    46,   46}
\definecolor{zip9}{RGB}{255,    48,   48}
\definecolor{zip10}{RGB}{255,   192,  203}

\usepackage{amsthm}

\newcommand{\citep}{\cite}
\newcommand{\citet}{\cite}
\newcommand{\citeyear}{\cite}
\newcommand{\citeauthor}{\cite}

\newtheorem{thm}{Theorem}
\newtheorem{defn}[thm]{Definition}

\newtheorem{rem}[thm]{Remark}
\newtheorem{res}[thm]{Result}

\usepackage{algorithm,algorithmic}

\usepackage{array,booktabs,longtable,tabularx}
\newcolumntype{L}{>{\raggedright\arraybackslash}X}
\usepackage{siunitx}
\usepackage{caption}

\usepackage{bm}
\usepackage{amsfonts}

\definecolor{ggplot1}{HTML}{E495A5}
\definecolor{ggplot2}{HTML}{BDAB66}
\definecolor{ggplot3}{HTML}{65BC8C}
\definecolor{ggplot4}{HTML}{55B8D0}
\definecolor{ggplot5}{HTML}{C29DDE}

\newcommand{\abs}[1]{|#1|}
\newcommand{\norm}[1]{\Vert#1\Vert}

\newcommand{\bLambda}{\boldsymbol{\Lambda}}

\newcommand{\bPhi}{\boldsymbol{\Phi}}

\newcommand{\bXi}{\boldsymbol{\Xi}}

\newcommand{\bGamma}{\boldsymbol{\Gamma}}

\newcommand{\bSigma}{\boldsymbol{\Sigma}}

\newcommand{\bPsi}{\boldsymbol{\Psi}}

\newcommand{\bOmega}{\boldsymbol{\Omega}}

\newcommand{\bB}{\boldsymbol{B}}

\newcommand{\bV}{\boldsymbol{V}}

\newcommand{\bI}{\boldsymbol{I}}

\newcommand{\bT}{\boldsymbol{T}}

\newcommand{\be}{\boldsymbol{e}}

\newcommand{\bP}{\boldsymbol{P}}

\newcommand{\bQ}{\boldsymbol{Q}}

\newcommand{\bW}{\boldsymbol{W}}
\newcommand{\bu}{\boldsymbol{u}}
\newcommand{\bv}{\boldsymbol{v}}

\newcommand{\bU}{\boldsymbol{U}}
\newcommand{\bx}{\boldsymbol{x}}
\newcommand{\bX}{\boldsymbol{X}}

\newcommand{\bY}{\boldsymbol{Y}}

\newcommand{\bw}{\boldsymbol{w}}

\newcommand{\bzero}{\boldsymbol{0}}

\newcommand{\bone}{\boldsymbol{1}}

\newcommand{\mR}{\mathcal R}

\newcommand{\mN}{\mathcal N}

\DeclareMathOperator*{\argmin}{argmin}

\DeclareMathOperator*{\trace}{trace}

\newcommand{\R}{\mathbb{R}}
\newcommand{\mS}{\mathbb{S}}
\newcommand{\B}{\mathbb{B}}
\newcommand{\mP}{\mathbb{P}}
\newcommand{\Var}{\mathbb{V}\mbox{ar}}

\newcommand{\Cor}{\mathbb{C}\mbox{or}}

\newcommand{\ben}{\begin{enumerate}}
\newcommand{\een}{\end{enumerate}}

\makeatletter
\newcommand\code{\bgroup\@makeother\_\@makeother\~\@makeother\$\@codex}
\def\@codex#1{{\normalfont\ttfamily\hyphenchar\font=-1 #1}\egroup}
\makeatother

\begin{document}
%
\title{Visualization of Labeled Mixed-featured Datasets}

\author{{Yifan~Zhu, Fan~Dai
     and Ranjan~Maitra}
\thanks{Y. Zhu and R. Maitra are with the Department of Statistics
at Iowa State University, Ames, Iowa 50011, USA. e-mail:
\{yifanzhu,maitra\}@iastate.edu.}
\thanks{F. Dai is with the Department of Mathematical Sciences at the Michigan Technological University, Houghton,
  Michigan 49931, USA. e-mail: fand@mtu.edu.}
 }

\maketitle
 \begin{abstract}
We develop methodology for visualization of labeled mixed-featured
datasets. We first investigate datasets with continuous features where
our Max-Ratio  
Projection (MRP) method utilizes the group information in high
dimensions to provide distinctive lower-dimensional projections that
are then displayed using Radviz3D. Our methodology is extended to
datasets with discrete and continuous features where a Gaussianized
distributional transform is used in conjunction with copula models
before applying MRP and visualizing the result using RadViz3D. A R
package {\tt radviz3d} implementing our complete methodology is available. 
\end{abstract}

%

%
\vspace{-0.2em}
\section{Introduction}
\label{sec:intro}
%
%
%
%
\IEEEPARstart{M}{odern} applications 
often yield datasets of many dimensions and complexity. Visualizing
such data is important to gain insight into their properties and the
similarity or distinctiveness of different
groups~\cite{cardetal99}. However, effective visualization can be
challenging because the observations need to be mapped  to a
lower-dimensional space, with the reduced display conveying
information on the characteristics as faithfully  as possible. Such displays become even more difficult with mixed-features data, that is, when some of the attributes in the dataset are discrete. 

Our major objective in this paper is to visualize high-dimensional datasets with mixed
features. 
Section~\ref{subsec:highd} first develops a Max-Ratio Projection (MRP)
method that linearly projects a labeled continuous-features dataset into a
lower-dimensional space to allow  for its effective visualization via RadViz3D
while preserving its group-specific distinctiveness and variability.
Our methodology is then extended to provide novel displays of datasets
that also have discrete features. Specifically, we use  the
Gaussianized Distributional Transform (GDT) with copula
models to render mixed-features datasets to the continuous space, after which MRP and RadViz3D can
be used.

Continuous multivariate data are    displayed in many ways~\cite{bertinietal11}
({\em e.g.} starplots~\cite{chambersetal83}, Chernoff 
faces~\cite{chernoff73}, parallel coordinate
plots~\citep{inselberg85,wegman90}, surveyplots~\cite{fayyadetal01}, Andrews'
curves~\cite{andrews72,khattreeandnaik02}, biplots~\cite{gabriel71}, star coordinate
plots~\cite{kandogan01}
, Uniform Manifold Approximation and Projections (UMAP)~\cite{mcinnesetal18}). Our paper applies
radial visualization or
RadViz~\cite{hoffmanetal97,hoffmanetal99,grinsteinetal01,draperetal09}
that projects data onto a circle using Hooke's law. Here, 
$p$-dimensional observations are projected onto the 2D plane using $p$ anchor
points equally arranged to be on the perimeter of a circle. This
representation places each observation at the center of the circle
that is then pulled by springs in the directions of the $p$ anchor
points while being balanced by forces relative to the coordinate
values. Observations with similar relative values across all
attributes are then placed close to the center while the others are
placed closer to anchor points corresponding to the coordinates with
higher relative values.  However, there is loss of
information~\cite{arteroanddeoliveira04} in RadViz which maps a
$p$-dimensional point to 2D.
This loss worsens with increasing $p$, but can be
alleviated by extending it to 3D \cite{zhu2021fully}.



Beyond Section~\ref{sec:method}, the paper is organized as
follows. Section~\ref{sec:illustration} illustrates the ability of 
our methodology to faithfully display labeled data of different
separations. Section~\ref{sec:app}  
illustrates our methodology on several high-dimensional datasets with
mixed features. The main paper concludes with some discussion in
Section~\ref{sec:conclusion}. 
We also have supplementary materials that discuss available methods
for displaying high-dimensional datasets, and an online resource at
\url{https://fanne-stat.github.io/RadViz3DExperiments/index.html} that 
allows for the reader to visualize displays in 3D. 
Items in the online resource are referenced in this paper with the prefix ``S''. 
\section{Methodology}
\label{sec:method}
\subsection{Visualizing  Discrete- and Mixed-Feature Datasets}
\label{subsec:discrete}
Datasets with discrete features are complicated
to visualize, but arise in genomics, survey and
voting preferences and other applications. We develop visualization
methods by transforming these datasets using copulas specifically
constructed to describe the correlation structure among the discrete
variables in the joint distribution while maintaining the empirical marginal
distribution. 
After transformation, we apply multivariate visualization methods. We transform mixed-feature datasets using copulas,
for which we introduce the generalized distributional transformation.  

\begin{defn}[Generalized Distributional Transform,
  \cite{ruschendorf13}, Chapter 1]\label{def1} 
 	Let $Y$ be a real-valued random variable (RV) with cumulative distribution
        function (CDF)
        $F(\cdot)$ and let $V$ be a RV independent of $Y$, such
        that $V \sim \mathrm{Uniform}(0,1)$. The generalized
        distributional transform of $Y$ is  	$U = F(Y, V)$
where 
$F(y, \lambda) \doteq P(Y < y) + \lambda P(Y = y) = F(y-) + \lambda\{F(y) - F(y-)\}$
is  	the generalized CDF of $Y$, and $F(y-)$ is the left limit of $F(\cdot)$ at $y$.
	\end{defn}

	\begin{thm}[\cite{ruschendorf13}, Chapter 1]\label{distributional_t}
	Let $U = F(Y, V)$ be the distributional transform of $Y$ as per
        Definition~\ref{def1}. Then
	\[U \sim \mathrm{Uniform}(0,1) \, \, \mathrm{and}\,\, Y = F^{-1}(U) \, \mathrm{a.s.} \] 
	where $F^{-1}(t) = \inf\{y \in \mR : F(y) \geq t\}$ 
	is the generalized inverse, or the quantile transform, of
        $F(\cdot)$.
	\end{thm}
	
	Suppose that $\bY_1,\bY_2,\ldots,\bY_n$ is a sample of
        discrete-valued random vectors, each of which has the same
        distribution as $\bY_1 = (Y_{11}, Y_{12}, \ldots, Y_{1p})$, where each margin $Y_{1i}$ has CDF $F_i(\cdot)$ (a step function).
	Let $U_i = F(Y_{1i}, V_i)$. From Theorem
        \ref{distributional_t}, $U_i \sim \mathrm{Uniform}(0,1)$,  so
        $(U_1, U_2, \ldots, U_p) \sim C$ is a copula. Further, from the
        definition of a quantile transform and  Theorem
        \ref{distributional_t}, the joint distribution for $\bY_1$ can
        be specified in terms of $F_i(\cdot)$s and the
        constructed copula $C$ as
	\begin{align*}
	    F(y_1, y_2,\ldots, y_p) &= \mP(Y_{11} \leq y_1, Y_{12}\leq y_2,\ldots, Y_{1p} \leq y_p)\\
	    & = \mP[F_i^{-1}(U_i)\leq y_i\;\forall i=1,2,\ldots,p]\\
	    & = \mP[U_i \leq F_i(y_i)\; \forall i=1,2,\ldots,p]\\
	    & = C[F_1(y_1), F_2(y_2), \ldots, F_p(y_p)]. 
	\end{align*}
        Now, we pick $p$ continuous marginal distributions, each with CDF
	$\tilde{F}_i(\cdot),i=1,2,\ldots,p$. Then $(\tilde{F}_1^{-1}(U_1),
	\tilde{F}_2^{-1}(U_2), \ldots, \tilde{F}_p^{-1}(U_p))$ has a continuous
	joint distribution with marginals $\tilde{F}_i(\cdot),\, i = 1, 2, \ldots, p$,
	and the copula associated with this joint distribution is also $C$.

        We use the marginal empirical CDF (ECDF) $\hat
             F_i(\cdot)$ of the  $\bY_j$s to estimate $F_i(\cdot)$ for
             $i=1,2,\ldots,p$. We use $\mN(0,1)$ as the continuous
             marginals, i.e. $\tilde{F}_i(\cdot) = \Phi(\cdot)$, the
             $N(0,1)$ CDF. We  define the Gaussianized distributional
             transform (GDT)
             \begin{equation}\label{Gfunction}
               G( \bY_j, \bV_j) \doteq [[\Phi^{-1}(\hat{F}_i(Y_{ji}, V_{ji}))]]_{i=1,2,\ldots,p}
      \end{equation}
       for $j=1,2,\ldots,n$. Here $\bV_j = (V_{j1},
      V_{j2},\ldots,V_{jp})$, and $V_{ji}$s are independent identically
      distributed  standard uniform realizations.
      Then $\bX_i = G(\bY_i, \bV_i),\, i = 1, 2, \ldots, n$ are
      realizations from a distribution on $\mR^p$:
      we apply the methods of Section~\ref{subsec:highd}  on
      $\bX_1,\bX_2,\ldots,\bX_n$ before visualizing the resulting MRPs
      using RadViz3D.
      \begin{rem} 
        \label{rem:gdt}
We make a few comments on our use of the GDT:
        \begin{enumerate}
        \item 
          \label{rem:cdf}
          For a continuous random variable, Theorem~\ref{distributional_t} reduces to the
          usual CDF so  $G(\cdot,\cdot)$ can be applied also to
          datasets with mixed (continuous and discrete) features.
            \item 
              \label{rem:std}
The GDT is a more stringent standardization than 
the usual affine transformation that only sets a dataset to have zero
mean and unit variance, because it transforms the 
marginal ECDFs to $\bPhi(\cdot)$. So the GDT may,
also be applied to datasets with skewed continuous features.
\item When datasets have discrete features with little
              class-discriminating ability, applying the
              GDT on such features will inflate the variance
              in the transformed space, resulting in a standard normal
              coordinate that is independent of the
              other features. When the number of such coordinates is
              substantial relative to group-discriminating features,
              these independent $\mN(0,1)$-transformed coordinates will
              drive the MRP, resulting in poor separation. So we use an 
              analysis of variance (ANOVA) test on each  
              copula-transformed coordinate to ascertain if it contains
              significant group-discriminating information. Multiple significance issues are addressed by correcting for false
              discoveries~\citep{benjaminiandhochberg95}. Features so 
              ascertained to not have significant disciminating
              information are 
              dropped from the MRP and subsequent steps.
            \end{enumerate}
      \end{rem}

\subsection{Visualizing High-dimensional Datasets}
\label{subsec:highd}
With the machinery for GDT in place, we  now investigate methods for summarizing labeled high-dimensional data with mixed features.
For even moderate dimensions ($p\!>\!10$), displaying many coordinates is not helpful even after 
factoring in the benefits of going to 3D. 
So we project our high-dimensional
datasets into a lower-dimensional space such that the projected
coordinates are almost uncorrelated.  A common approach to finding
uncorrelated projections is Principal Components Analysis (PCA) that
finds the mutually orthogonal projections 
summarizing a proportion of the total variance in the data. We propose
using PCA for unlabeled data. For labeled data, approaches beyond PCA
that exploit  class labels~\citep{korenandcarmel04} are desired. Our
objective is to find an approach that preserves the distinctiveness
of group labels when finding projections while also preserving, in
the display, the inherent variability in the dataset. 
We develop Max-Ratio Projections (MRPs) of the data that maximizes separation
between groups (in projected space) relative to its total variability. 
We discuss obtaining these projections next.
\paragraph{\bf Directions that Maximize Between-Group
  Variance}
Given a labeled dataset, we find a linear subspace such that the
groups are well-separated when the data are projected along  this
subspace. Let $\bv_1, \bv_2,\ldots, 
\bv_{k}$ be $k$ uncorrelated direction vectors spanning the linear
subspace. In order to separate the groups, we want to project the 
data to each $\bv_j$ such that the ratio of the projected between-group
sum of squares and the total corrected sum of squares is maximized
(equivalently, the ratio of the projected within-group sums of squares
and the total corrected sum of squares is minimized). 

Let $\bXi=\{\bX_1,\bX_2,\ldots,\bX_n\}$ be $n$ $p$-dimensional observation
vectors. Then the corrected total sum of squares and cross-products
(SSCP) matrix is  $\bT = (n-1)\hat\bSigma$ where
$\hat\bSigma$ is the sample dispersion matrix of $\bXi$. Let
$\Var(\bX_i) = \bSigma$, then for any 
projection vector $\bv_j$, we have $\Var(\bv_j'\bX_i) =
\bv_j'\bSigma\bv_j$. Further, for any two $\bv_j$ and $\bv_l$,
$\Cor(\bv_j'\bX_i,\bv_l'\bX_i) \propto \bv_j'\bSigma\bv_l = 0$ since
the direction vectors decorrelate the observed coordinates. (We may 
replace $\bSigma$ with $\hat\bSigma$ in the expressions above.)  Therefore,
we obtain $\bv_1,\bv_2,\ldots,\bv_k$ in sequence to satisfy
\begin{equation}
\begin{split}
    &\max_{\bv_1} \frac{SS_{group}(\bv_1)}{SS_{total}(\bv_1)}\\
    &\max_{\bv_j} \frac{SS_{group}(\bv_j)}{SS_{total}(\bv_j)} \quad
    \ni\quad \bv_j' \bT \bv_i = 0,\, 1\leq i < j\leq k
\end{split}
\label{eq:seq.v}
\end{equation}
where $SS_{total}(\bv_l)$ is the corrected total sum of squares of the data
projected along $\bv_l$ (so is a scalar quantity), and
$SS_{group}(\bv_l)$ is the corrected between-group sum of squares of the data
projected along $\bv_l$.  
Equivalently, if $SS_{within}(\bv_l)$ is the corrected within-group
sum of squares of data projected along $\bv_l$, we know  
$SS_{total}(\bv_l) = SS_{group}(\bv_l) + SS_{within}(\bv_l).$
Therefore the ratio can be written as
\[\frac{SS_{within}(\bv_l)}{SS_{total}(\bv_l)}
  =1 - \frac{SS_{group}(\bv_l)}{SS_{total}(\bv_l)}.\]
Then the equivalent form of the optimization problem is 
\begin{equation}
\begin{split}
    &\min_{\bv_1} \frac{SS_{within}(\bv_1)}{SS_{total}(\bv_1)}\\
    &\min_{\bv_j} \frac{SS_{within}(\bv_j)}{SS_{total}(\bv_j)} \quad
    \ni\quad \bv_j' \bT \bv_i = 0,\, 1\leq i < j\leq k.
\end{split}
\label{eq:min.seq.v}
\end{equation}
\begin{thm} {\em Max-Ratio Projections}. Let
  $\bX_1,\bX_2,\ldots,\bX_n$ be $p$-dimensional observations from $G$
  groups. Let $\bT$ be   the total corrected SSCP and $\bB$ be the
  corrected SSCP between groups. Let $\bT$ and $\bB$  both be
  positive definite. Then
  \begin{equation}
    \hat\bv_j = \frac{\bT^{-\frac12}\hat\bw_j}{\norm{\bT^{-\frac12}\hat\bw_j}},
    \qquad j=1,2,\ldots,k
    \label{eq:vj}
  \end{equation}
satisfies \eqref{eq:seq.v}  where $\hat\bw_j, j = 1, 2,\ldots,k$ are the eigenvectors corresponding to the $k$ largest eigenvalues of $\bT^{-1/2}\bB\bT^{-1/2}$.
\label{res:mrp}
 \end{thm}
 \begin{proof}
 Let $\bGamma_g,g=1,2,\ldots,G$ be the $n_g\times n$ matrix that
 selects observations from the matrix $\bX$ that has $\bX_i$ as its $i$th row. Here $n_g$ is the number
 of observations from the $g$th group, for $g=1,2,\ldots,G$. Then
 $\bGamma_g\bX$ is the matrix with observations
 from the $g$th group in its rows and
 \begin{equation}
   \bB = \bX'\left(\sum_{g=1}^G\frac1{n_g}\bGamma_g'\bone_{n_g}\bone_{n_g}'\bGamma_g
   -\frac1n\bone_n\bone_n'\right)\bX
 \label{eq:B}
 \end{equation}
 and $\bT = \bX'(\bI_n -\bone_n\bone_n'/n)\bX$. Also,
 $\bX$ projected along any direction $\bv$ yields $SS_{group}(\bv) =
 \bv'\bB\bv$ and $SS_{total}(\bv) = \bv'\bT\bv$ so that 
 finding~\eqref{eq:seq.v} is equivalent to 
 \begin{equation}
 \begin{split}
   &\max_{\bv_1} \frac{ \bv_1'\bB\bv_1}{\bv_1'\bT\bv_1} \\
   &\max_{\bv_j} \frac{  \bv_j'\bB\bv_j}{\bv_j'\bT\bv_j}  \quad \bv_j'
   \bT \bv_i = 0,\,  1\leq i<j \leq k.
 \end{split}
 \label{eq:seq.v.new}
 \end{equation}
 Let $\bw_j = \bT^{1/2} \bv_j,\, j = 1,2, \ldots, k$. Then for each $j$,
 \begin{equation*}
 \frac{SS_{group}(\bv_j)}{SS_{total}(\bv_j)} = \frac{\bv_j' \bB \bv_j}{\bv_j'\bT \bv_j} = \frac{\bw_j' T^{-\frac12}B T^{-\frac12} \bw_j}{\bw_j' \bw_j}
 \end{equation*}
 	and $\bv_j'\bT\bv_i = \bw_j' \bw_i $. Then, instead
        of~\eqref{eq:seq.v.new}, we can sequentially solve the
        following, with respect to $\bw_1,\bw_2, \ldots, \bw_k$: 
 	\begin{equation}
 \begin{split}
     &\max_{\bw_1} \frac{\bw_1' \bT^{-\frac12}\bB\bT^{-\frac12} \bw_1}{\bw_1'\bw_1}\\
    &\max_{\bw_j} \frac{\bw_j' \bT^{-\frac12}\bB\bT^{-\frac12}
       \bw_j}{\bw_j' \bw_j} \; \ni \; \bw_j' \bw_i = 0,\, 1\leq
     i < j \leq k. 
 \end{split}
 \label{eq:svd}
 \end{equation}
 $\bT^{-\frac12}\bB\bT^{-\frac12}$ is nonnegative definite, with at
 most $G-1$ positive eigenvalues, so $k\leq G-1$ in
 \eqref{eq:svd}. Its eigenvectors $\hat\bw_1,\hat\bw_2, \ldots, \hat\bw_k$ 
(corresponding to its $k$ largest eigenvalues in decreasing order) solve 
 \eqref{eq:svd}. Let 
 $\hat\bv_j$ be the normalized version of $\bT^{-\frac12}\hat\bw_j$. Then
 $\hat\bv_j$s satisfy \eqref{eq:seq.v.new}. The theorem follows.
 \end{proof}
Theorem~\ref{res:mrp} provides the projections that maximize
separation between the groups in a lower-dimensional space in a
way that also decorrelates the coordinates. The number of
projections is limited by $G-1$. So, for $G\leq 3$, 1 to 2 projections and
therefore 1D or 2D displays should be enough. (For $G=3$, a 
RadViz2D figure should normally suffice, but, as we show later in
our examples, choosing 4 projections yields a better display even
though the additional $4-G+1$ projections yield no additional
information on group separation. We use springs to
provide a physical interpretation for why these additional 
$4-G+1$ coordinates are beneficial. The first $G-1$ MRP coordinates
pull the data with different forces along the corresponding anchor
points in a way that permits maximum separation of the classes. The 
remaining $4-G+1$ anchor points correspond to the zero eigenvalues and
do not contribute to the separation between groups, and so each
group is pulled with equal force in the direction of these anchor
points. These additional pulls separate the groups 
better in RadViz3D than in RadViz2D. (We choose 4 MRPs when $G\leq4$
for RadViz3D because a 3D sphere is best separated using 4
equi-spaced anchor points because every axis is then equidistant
to the other.)

The eigenvalue decomposition of $\bT^{-\frac12}\bB\bT^{-\frac12}$ assumes a
positive definite $\bT$, for which a sufficient condition is that $n_g
> p\mbox{ for all } g$. This assumption may not always
hold so we now propose to reduce the dimensionality of the dataset for
the cases where $p \geq \min_g n_g$ while also preserving as far as possible its
group-specific features and variability.
\paragraph{\bf Nearest Projection Matrix to Group-Specific Principal Components (PCs)}
\label{sec:dir.MRP}
Our approach builds on standard
PCA whose goal, it may be recalled, is to project a dataset onto a
lower-dimensional subspace in a way that captures most of its total
variance. We use 
projections that summarize the variability within each
group. So, we summarize each group by obtaining PCs separately and then finding the closest projection
matrix to all the group-specific PCs. Specifically, we have the following
\begin{res}
\label{res3}
Suppose that $\bV_1, \bV_2,\ldots, \bV_m$ are $p\times q$ matrices
with $\bV_j'\bV_j = \bI_q$, where $\bI_q$ is the $q\times q$ identity
matrix. Let $\bV=\sum_{j=1}^m\bV_j$ with singular value
decomposition (SVD) $\bV=\bP_\bullet\bLambda_\bullet\bQ'$ where
$\bP_\bullet$ is a $p\times q$ matrix of orthogonal columns, $\bQ$
is a $q\times q$ orthogonal matrix and 
$\bLambda_\bullet$ is a $q\times q$ diagonal matrix with $v$ non-zero
entries where $v = \mbox{rank} (\bV)$. Then the $p\times q$ matrix $\bW =
\bP_\bullet\bQ'$ satisfies 
\begin{equation}
\bW = {\argmin}\{ \sum_{j=1}^{m}\norm{ \bW - \bV_j}_{F}^{2}: \bW'\bW
= \bI_q\}.
\label{eq:opt}
\end{equation}
\end{res}
 \begin{proof} (Several proofs for Result~\ref{res3}~\cite{golub1996}
   exist but we provide a novel alternative proof to add to the literature.)
 Minimizing $\sum_{j=1}^{m}\norm{ \bW - \bV_j}_{F}^{2}$ is the same as
 maximizing $\sum_{j=1}^{m}\trace{(\bW'\bV_j)}$ or, equivalently, 
 $\trace{(\bW'\bV)}$. Let the full SVD of $\bV
 =[\bP_\bullet,\bP_\circ][\bLambda_\bullet,\bzero]'\bQ'$, where
 the $i$th diagonal element of $\bLambda_\bullet$ is the nonnegative
 eigenvalue $\lambda_i$. Then $\trace{(\bW'\bV)} =
 \trace{(\bQ'\bW'[\bP_\bullet,\bP_\circ][\bLambda_\bullet,\bzero]')}
 $. Let $\bB = \bQ'\bW'[\bP_\bullet,\bP_\circ]$ have $b_{ij}$ as
 its $(i,j)$th entry. Then $\bB\bB' = \bI_q$ and $|b_{ij}|\leq1$ for
 all $i,j$. So,  
 $
 \trace{(\bW'\bV)}\leq\abs{\trace{(\bW'\bV)}}=\abs{\sum_{i=1}^v\lambda_{i}
 b_{ii}}\leq\sum\lambda_{i} \abs{b_{ii}}
 \leq\sum\lambda_{i} = \trace{(\bLambda_\bullet)}$,
 with equality holding when $\bW =
 \bP_\bullet\bQ'$.
 \end{proof}
Result~\ref{res3} reduces dimensionality of a dataset for when the number of
features is larger than the minimum number of records in any
group. We take $m = \min\{p,n_1,n_2,\ldots,n_g\}$. The $k$ MRPs of our
dataset are displayed using RadViz3D. The choice of $k$ may be based
on the clarity of the 
display, or by the cumulative proportion (we use 90\%)
of the eigenvalues  of $\bT^{-1/2}\bB\bT^{-1/2}$.  
\begin{rem} We compare MRP with orthogonal
  linear discriminant   analysis (OLDA)~\citep{ye05} and uncorrelated linear
  discriminant   analysis (ULDA)~\citep{jinetal01}. OLDA
  also produces orthogonal discriminant vectors that project data onto a
  lower-dimensional subspace, however the projection vectors satisfy
  $\bv_i'\bv_j = 0,\, i \neq j$. This does not necessarily mean
  uncorrelated projections of the data since $\bv_j'\bSigma\bv_i$ is not
  necessarily zero. In our  visualization, $v_i'\Sigma v_j =0,\, i
  \neq j$ is desired and MRP satisfies this by producing
  uncorrelated projection directions.  On the other hand, ULDA uses
  the same set of uncorrelated projection directions as MRP. However,
  MRP uses the normalized vectors (columns of projection matrix)
  with unit length, while the column vectors of the ULDA projection
  matrix are not normalized. Specifically, both ULDA and MRP have
  $v_i' \Sigma v_j = 0,\, i \neq j$, but in ULDA, $v_i' \Sigma v_i =
  1$, while in MRP, $v_i'v_i = 1$. So the ULDA is actually MRP with a
  scaling step after the projection, such that each coordinate's
  variance is 1. In that sense, MRP can display a better visualization
  as the variance of for each projection direction is preserved. 
\end{rem}
      Algorithm~\ref{alg1} summarizes the use of GDT and MRP
      for mixed-features datasets.
       \begin{algorithm}[H]
 \caption{RadViz3D for datasets with mixed features\label{alg1}}
 \begin{algorithmic}[1]
  \STATE Calculate the marginal ECDF $\hat{F}_1(\cdot), \hat{F}_2(\cdot), \ldots,
  \hat{F}_p(\cdot)$ for each of the $p$ coordinates of the dataset.
  \STATE Simulate $\bV_i \overset{iid}{\sim} \mathrm{Uniform}[0,1]^p$.
  \STATE Construct the transform $G(\cdot,\cdot)$ with  marginal ECDFs and
  simulated $\bV_i, i = 1,2, \ldots, n$, as in Equation~\ref{Gfunction}. 
  \STATE Transform $\bY_i$ in the discrete dataset to $\bX_i $ with $\bX_i = G(\bY_i, \bV_i),\, i = 1,2,\ldots, n$.
  \STATE Apply MRP on $\bX_i,\, i = 1,2,\ldots, n$ via Algorithm~\ref{alg1}.
  \STATE Display MRP results by RadViz3D.
 \end{algorithmic} 
 \end{algorithm} 

 \section{Illustrative Performance Evaluations}
 \label{sec:illustration}
We illustrate our methodology on 
simulated 100D datasets of $n=500$ observations from five groups with
both discrete and continuous features and of known group
separation and clustering complexity. The {\sc MixSim} 
package~\citep{melnykovetal12} in R\citep{R} allows for the simulation 
of (continuous) class data according to a pre-specified {\em
  generalized overlap}
($\ddot\omega$)~\citep{maitraandmelnykov10,melnykovandmaitra11} that
indexes clustering complexity, with very small values ($\ddot\omega =
0.001$) implying very good separation between groups and larger values
($\ddot\omega = 0.05$) indicating  poorer separation and increased
overlap. We discretize the first 50 coordinates in each group into 10
classes, based on the marginal deciles of the coordinate. 
We use the GDT and MRP on this mixed-features dataset and visualize using
RadViz3D. Because of how our mixed-features datasets were generated,
the {\sc MixSim}-estimated pairwise overlaps between the 5 groups are
essentially preserved (and displayed in 
the left columns of Fig.~\ref{fig:simcdr}). Two views of RadViz3D
displays of each dataset are in the middle two columns, and
a 3D UMAP~\citep{mcinnesetal18} visualization is in the right column
of Fig.~\ref{fig:simcdr}).  
 \begin{figure*}[h]
   \vspace{-0.2in}
   \mbox{
     \setcounter{subfigure}{-2}
     \subfloat[$\ddot\omega=0.001$]{
       \begin{minipage}[b][][t]{\textwidth} 
      \mbox{\subfloat{\includegraphics[width=.25\textwidth]{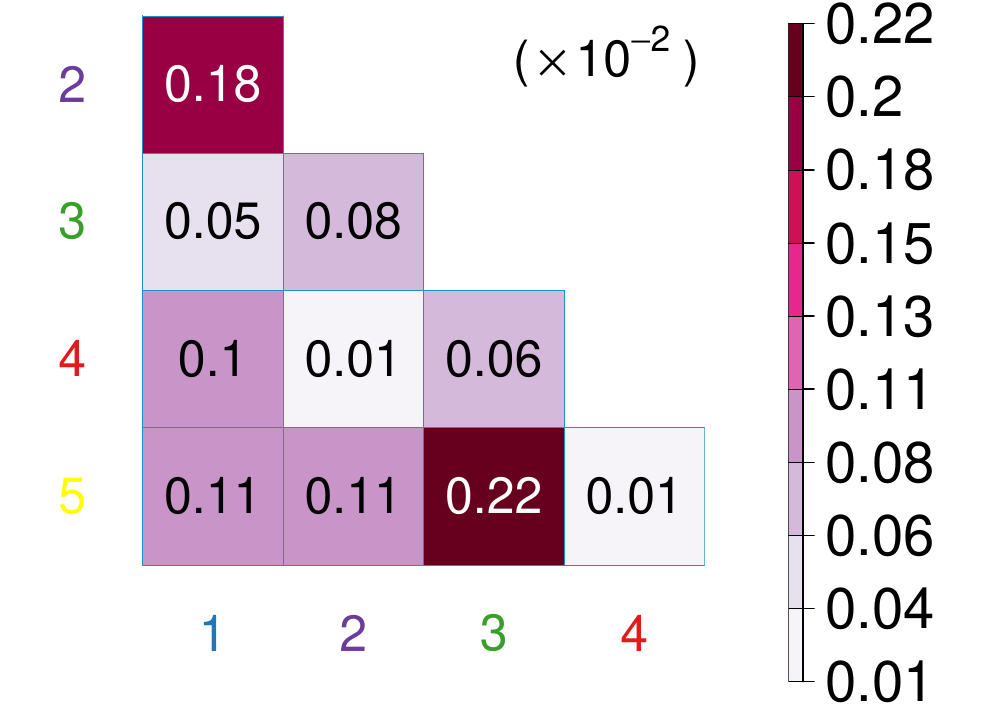}}
        \subfloat{\includegraphics[width=.25\textwidth]{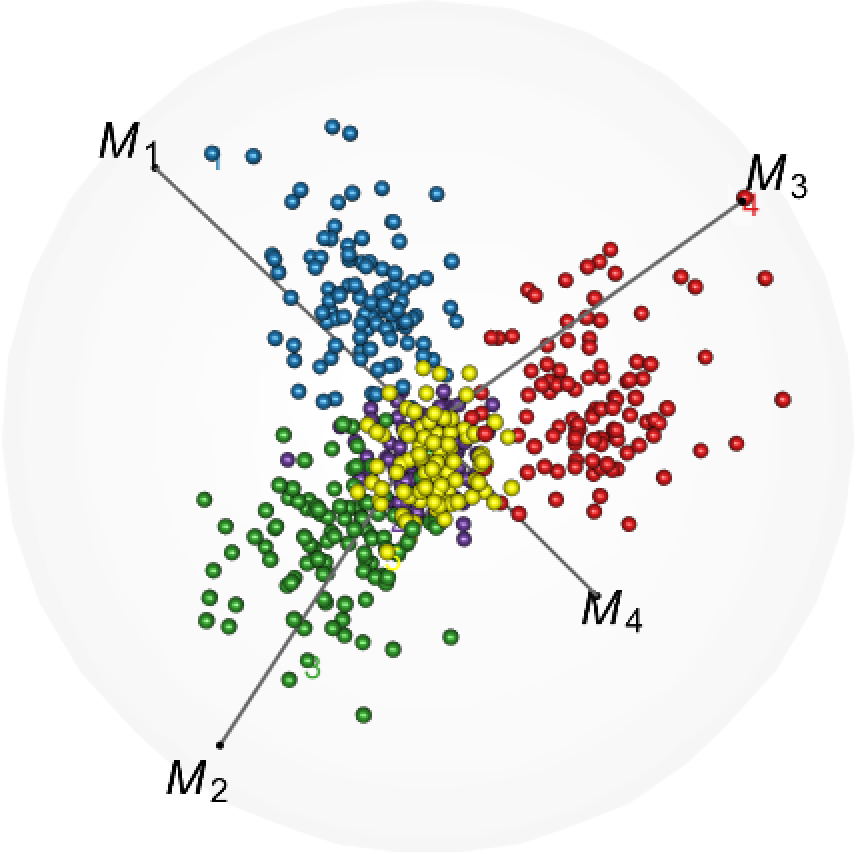}}
        \subfloat{\includegraphics[width=.25\textwidth]{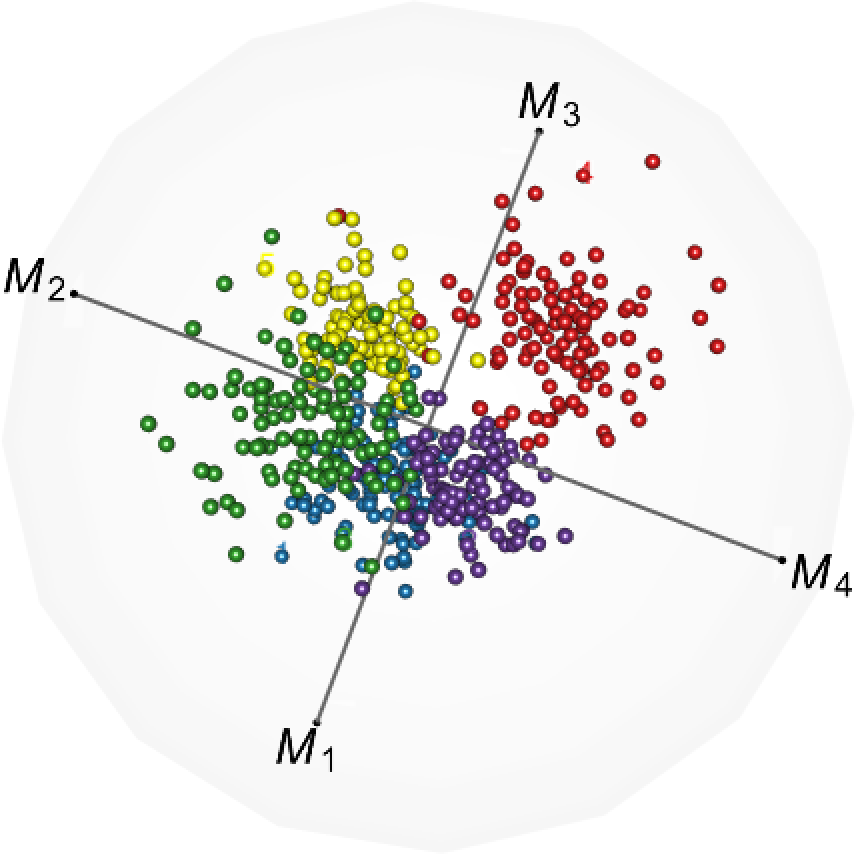}}
        \subfloat{\includegraphics[width=.25\textwidth]{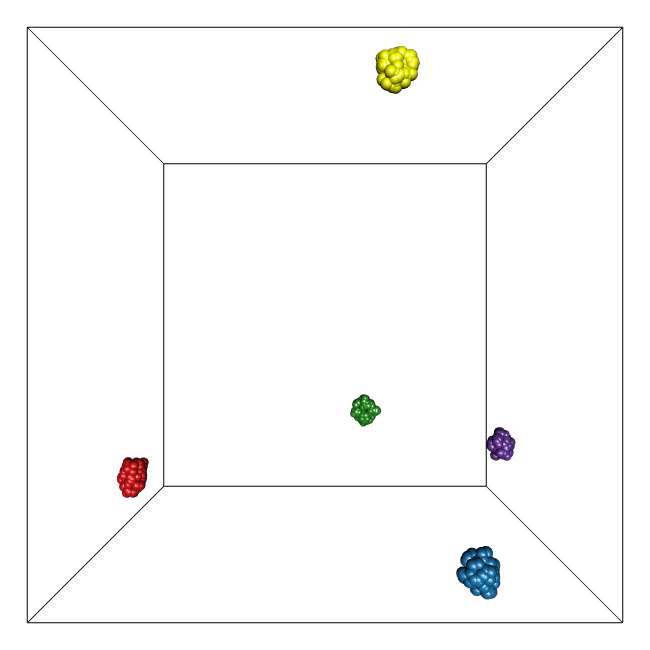}}
      }
\vspace{-0.1in}
      \end{minipage}}%
  }
   \mbox{
     \setcounter{subfigure}{-1}
     \subfloat[$\ddot\omega=0.01$]{
       \begin{minipage}[b][][t]{\textwidth} 
      \mbox{\subfloat{\includegraphics[width=.25\textwidth]{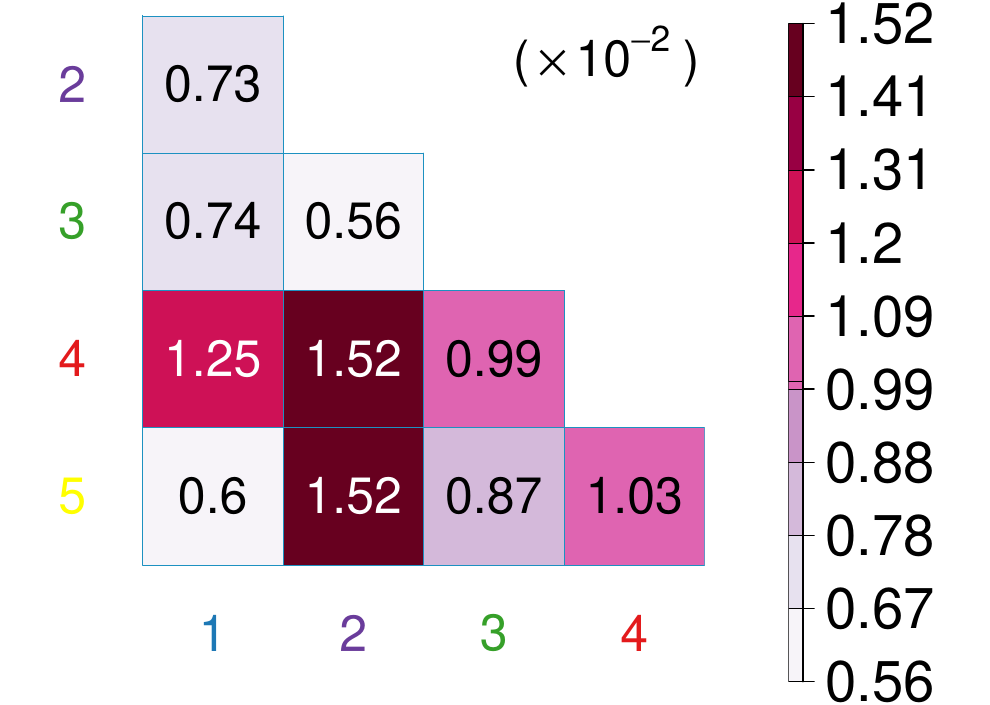}}
        \subfloat{\includegraphics[width=.25\textwidth]{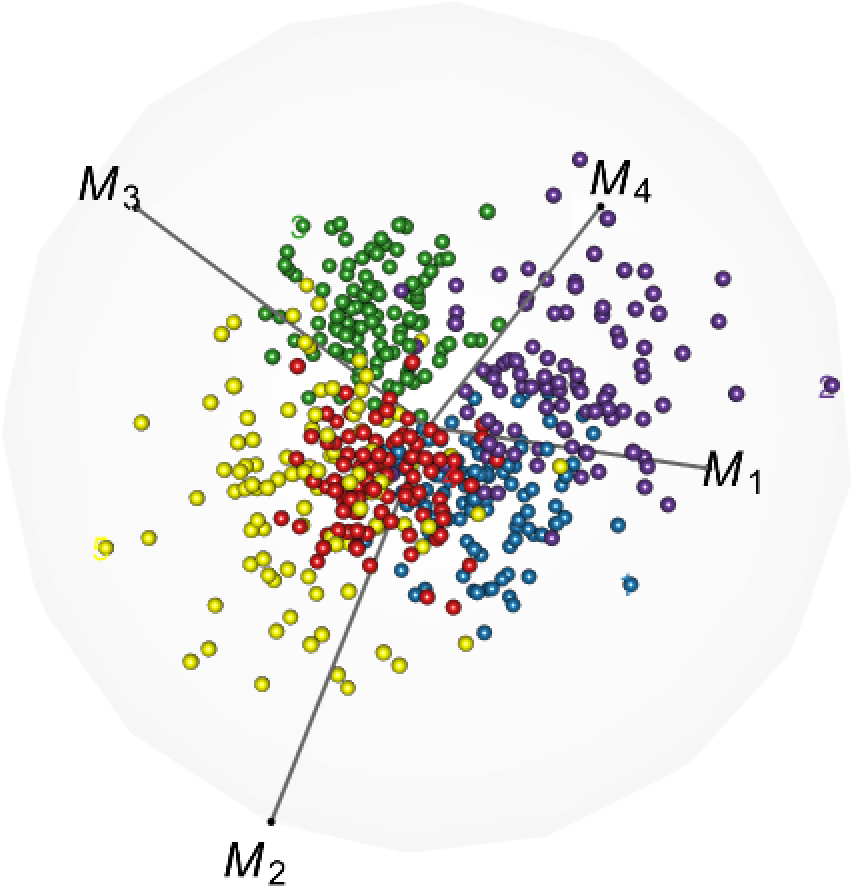}}
        \subfloat{\includegraphics[width=.25\textwidth]{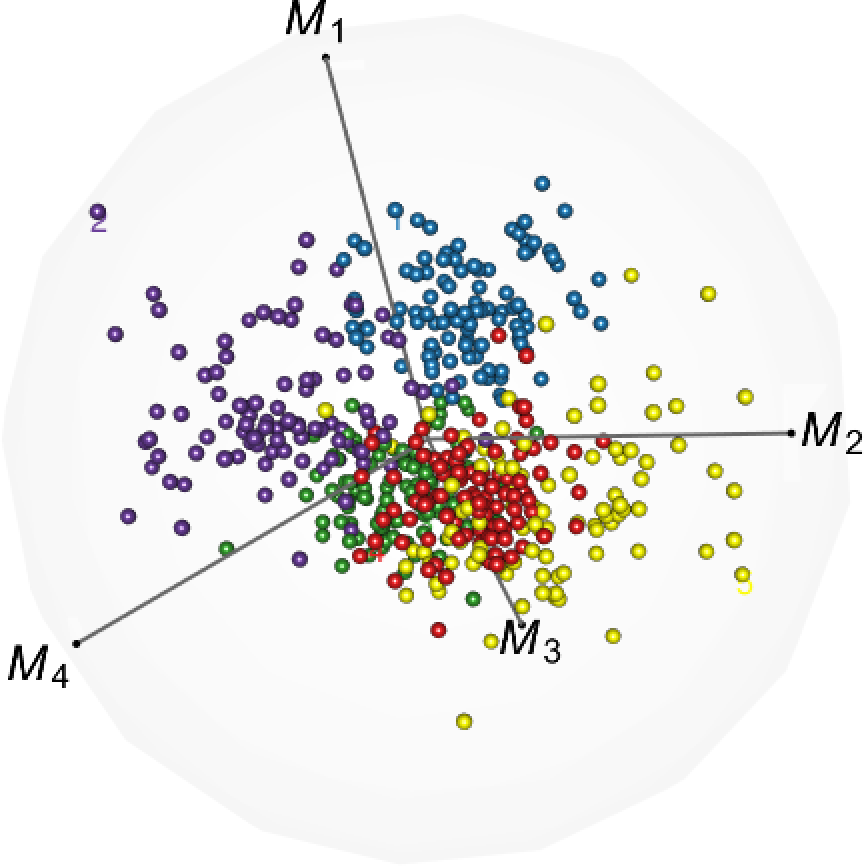}}
        \subfloat{\includegraphics[width=.25\textwidth]{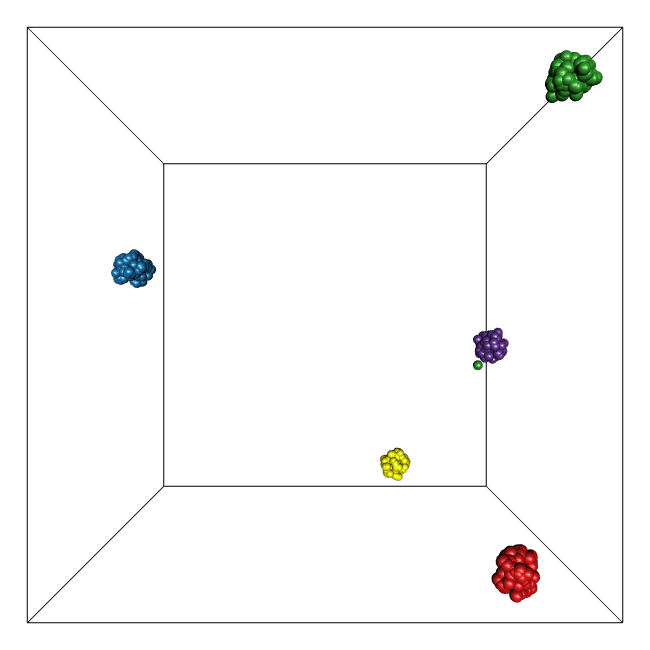}}
      }
\vspace{-0.1in}
      \end{minipage}}%
  }
   \mbox{
     \setcounter{subfigure}{-0}
     \subfloat[$\ddot\omega=0.05$]{
       \begin{minipage}[b][][t]{\textwidth} 
      \mbox{\subfloat{\includegraphics[width=.25\textwidth]{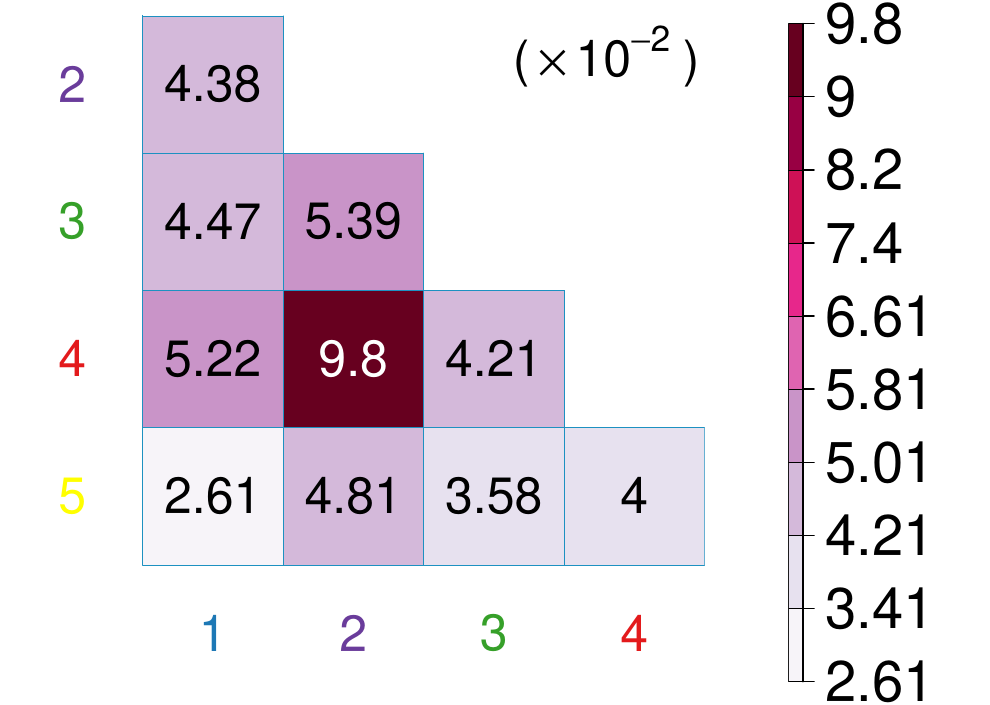}}
        \subfloat{\includegraphics[width=.25\textwidth]{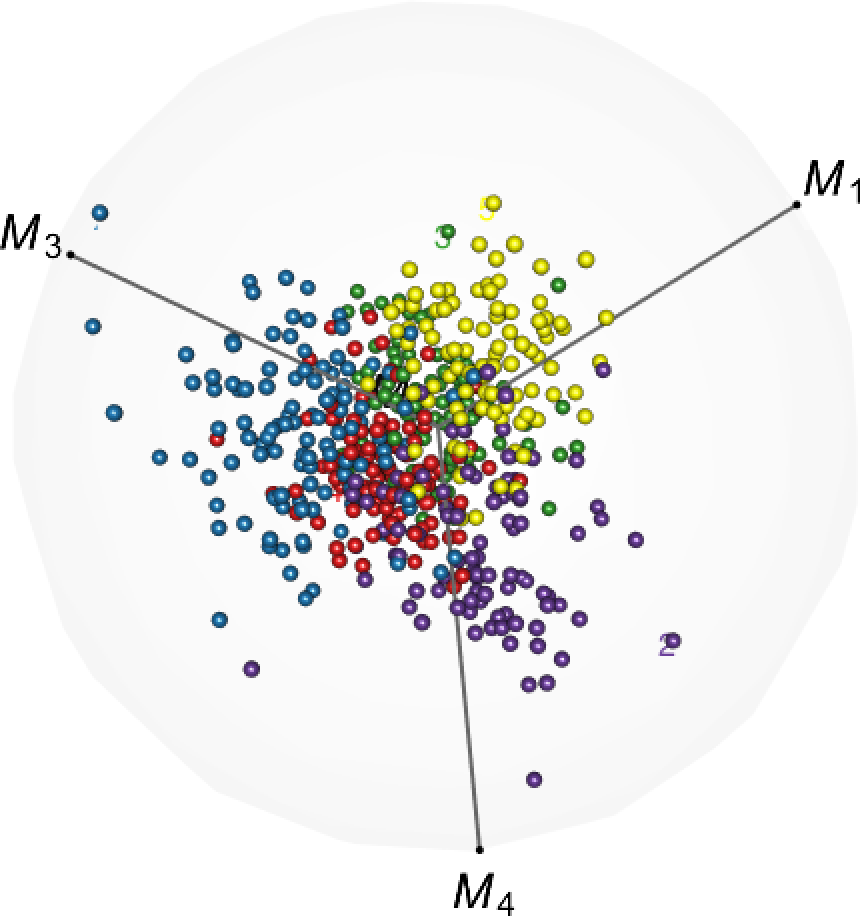}}
        \subfloat{\includegraphics[width=.25\textwidth]{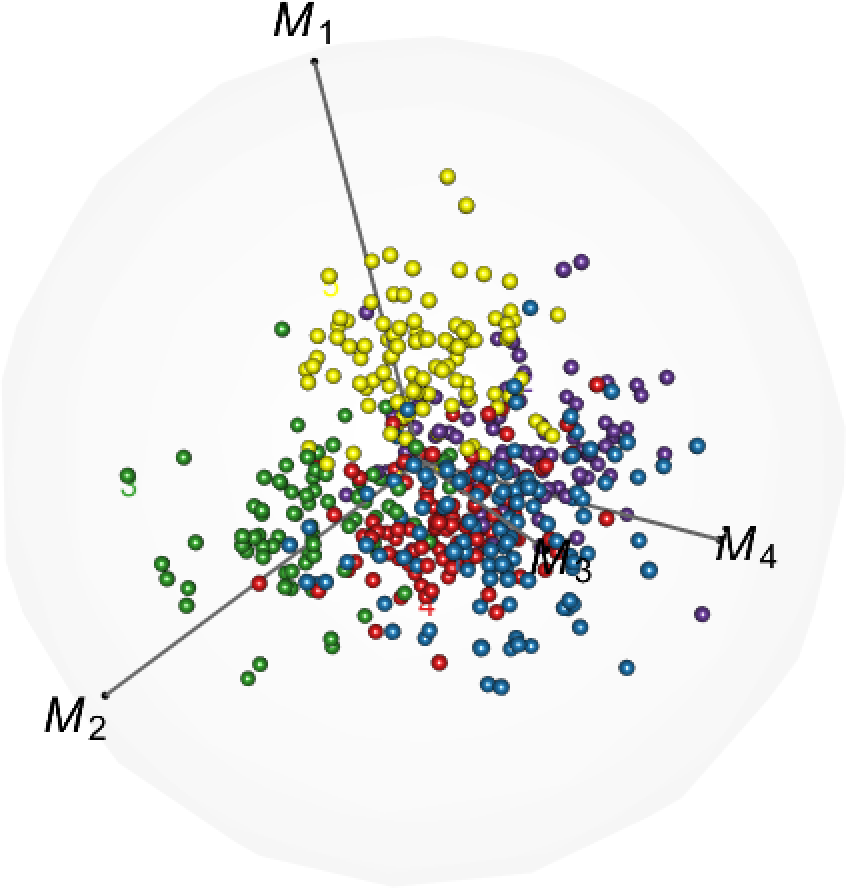}}
        \subfloat{\includegraphics[width=.25\textwidth]{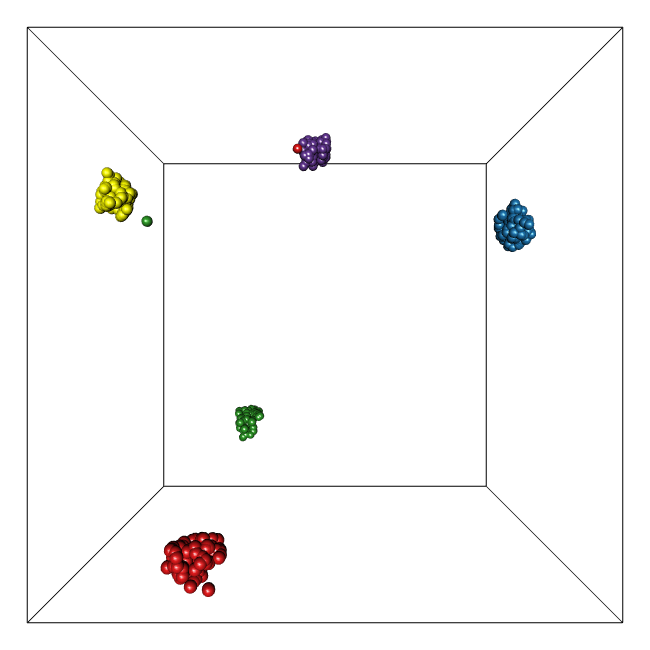}}
      }
\vspace{-0.1in}
      \end{minipage}}%
  }
\vspace{-.25cm}
\caption{Overlap maps ($\bOmega$, left) RadViz3D (middle figures) and UMAP displays (right) of simulated
  100D mixed-feature datasets of varying group separation ($\ddot
  \omega$). Label colors in $\bOmega$ match those in the RadViz3D displays.} 
\label{fig:simcdr}
\vspace{-.05cm}
\end{figure*}
Fig.~{S1} provides interactive 3D displays
of RadViz3D, UMAP, $t$-SNE, Viz3D, and star coordinate plots with OLDA and
ULDA. (OLDA,ULDA and MRP assume a low-dimensional linear subspace and require continuous variables while the default Euclidean distance used by UMAP and t-SNE is not always appropriate for mixed-features data. These methods are all also scale-variant. We use the GDT in mixed-features datasets -- itself a major contribution of the paper -- to address all these issues in all methods.) Further, in discussing the different displays, we note that the objective behind
accurate visualization of labeled data is the display of labels
according to the actual separation between them.

Fig.~\ref{fig:simcdr} and Fig. S1 show  that Radviz3D provides
meaningful displays that track the 
difficulty of separation very well as $\ddot\omega$ increases from
0.001 through 0.01 to 0.05 (see Fig.~S1 for dynamic displays for
$\ddot\omega\in\{0.0001,0.001,0.01,0.05, 0.25\}$. Further, Fig.~\ref{fig:simcdr}a shows the
highest overlap between Groups 3 and 5 and between Groups 1 and 2 is
fairly accurately reproduced in the RadViz3D display. Similar patterns are
also noticed in Figs.~\ref{fig:simcdr}b and \ref{fig:simcdr}c. For
instance, in Fig.~\ref{fig:simcdr}b, there are high overlaps between
Groups 2 and 4 and between Groups 2 and 5, and these high overlaps,
relative to the other pairs, is also reflected
in the RadViz3D display. 
From Fig.~\ref{fig:simcdr}c, we see that Groups 2 and 4 have the
highest overlap and again RadViz3D produces a display consistent with
this observation. Our illustrative experiment therefore shows
RadViz3D's ability to faithfully display high-dimensional grouped
datasets with varying separation.  
Our displays of the competing methods show that, but for UMAP, they
are unable to display the separation of classes even in the case when
they are generated with low $\ddot\omega$. UMAP (for instance,
Fig.~\ref{fig:simcdr} can separate out classes in all cases very well,
but can not distinguish the cases when class labels are well-separated
from the cases when they are not. Specifically, 
UMAP provides similar representation of the  distinctiveness of the
groups across all cases regardless of whether the groups are well- or
poorly-separated (as quantified by the overlap measures). 
We contend that this good separation of data by labels regardless of
their overlap or difficulty of separation is a desirable property  of 
a classification algorithm but not for visualization which should
faithfully render the correct status and should display well-separated
labeled data as such and poorer-separated labeled data as
such. In that goal, only RadViz3D provides meaningful displays. 

\section{Real-Data Examples}
\label{sec:app}
We now explore the visualization of datasets  with continuous, discrete
or mixed features. The focus of our work is on displaying
high-dimensional datasets, so we provide only one  moderate (9D)
example. Our other examples have larger $p$, in some cases of several
thousands. For brevity, we only have static displays here with RadViz3D,
and refer to the online resource 
for dynamic or competing displays. 
 \subsection{Datasets with discrete features}
 \label{app:discrete}
 We begin with datasets with discrete features. There exist
 no competing methods that can handle such datasets so we use the GDT 
 on these datasets before using them. The MRP is also used before
 displaying with RadViz3D, star coordinates and Viz3D displays.
 \subsubsection{Voting records of US senators}
 \label{sec:senators}
 \begin{figure}[h]
   \vspace{-0.2in}
   \centering
 \mbox{
 \subfloat[]{\label{senators:radviz3d1}\includegraphics[width=0.25\textwidth]{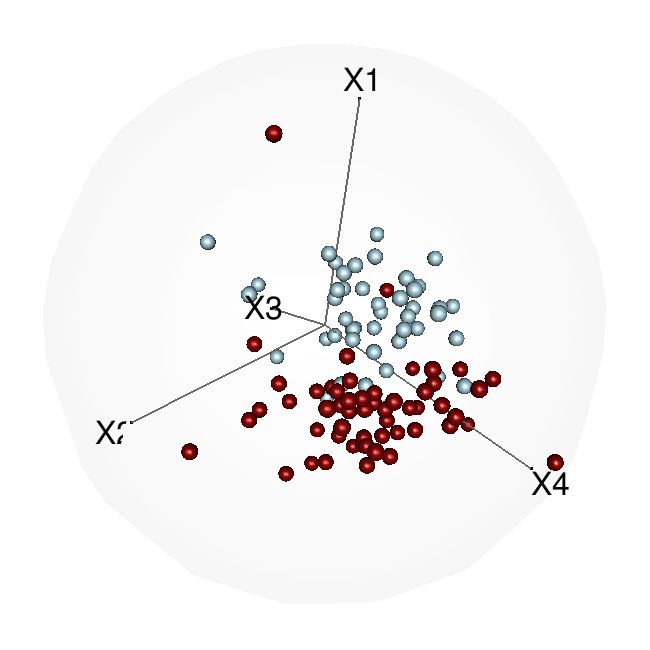}}
 \subfloat[]{\label{senators:radviz3d2}\includegraphics[width=0.25\textwidth]{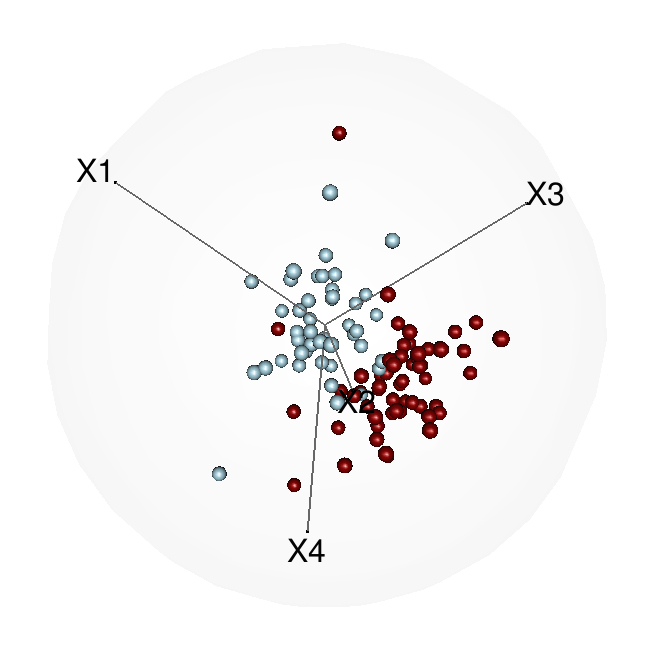}}
 }
 \par\bigskip
 \centering
 \text{
 \tikz\draw[black,fill=dkp1] (0,0) circle (.5ex); Democratic 
 \tikz\draw[black,fill=dkp2] (0,0) circle (.5ex); Republican
 }
 \caption{RadViz3D displays of senators' voting records.}
 \label{senators}
 \vspace{-0.1in}
 \end{figure}
 The 108th US Congress had 55 Republican and 45 Democratic (including
 1 independent in the Democratic caucus) senators vote on 542
 bills~\citep{banerjeeetal08}. We display the senators according to whether they voted for each bill or not ({\em i.e.}
 against/abstained). 
 The RadViz3D~(Fig.~\ref{senators}, Fig.~S6f) display 
 distinguishes the 2 groups, reflecting the political affiliation. 
 Here $G\!\!=\!\!2$, so three
 zero-eigenvalue projections beyond the MRP are used in the RadViz3D
 display. As explained in Section~\ref{sec:dir.MRP}, these additional
 projections (associated with the anchor  points $X_2,X_3,X_4$) do not
 contribute towards separating the two groups 
 which are separated solely by the first MRP (associated
 with $X_1$). A physical interpretation is that the spring on anchor
 point $X_1$ pulls one group harder than another group, separating it
 out, while the  ``null'' springs on $X_2, X_3, X_4$ pull both groups with equal
 force. All anchor points are evenly-distributed on the unit sphere,
       so the forces applied by $X_2, X_3, X_4$ cancel and only the
       spring attached to $X_1$ separates the first  group from the
       second in the visualization.
       RadViz3D and UMAP separate out
       the two groups, but RadViz3D, unlike UMAP (Fig.~S6b), also
       gives us a sense of the        closeness of some senators in
       either  group with the         other.  
The $t$-SNE~(Fig.~S6a), Viz3D~(Fig.~S6e) and especially star
coordinate plots with ULDA (Fig.~S6c) and OLDA (Fig.~S6d) are unable
to separate the two parties. 

 \subsubsection{Adult Autism Spectrum Disorder (ASD) screening}
 This dataset~\citep{thabtah17} from the UCI's Machine
 Learning Repository (MLR)~\citep{newmanetal98} has 15 binary (and 5
 additional) features on 515 normal and 189 ASD-diagnosed
 adults. 
 \begin{figure}[h]
   \centering
 \mbox{
 \subfloat[]{\includegraphics[width=0.25\textwidth]{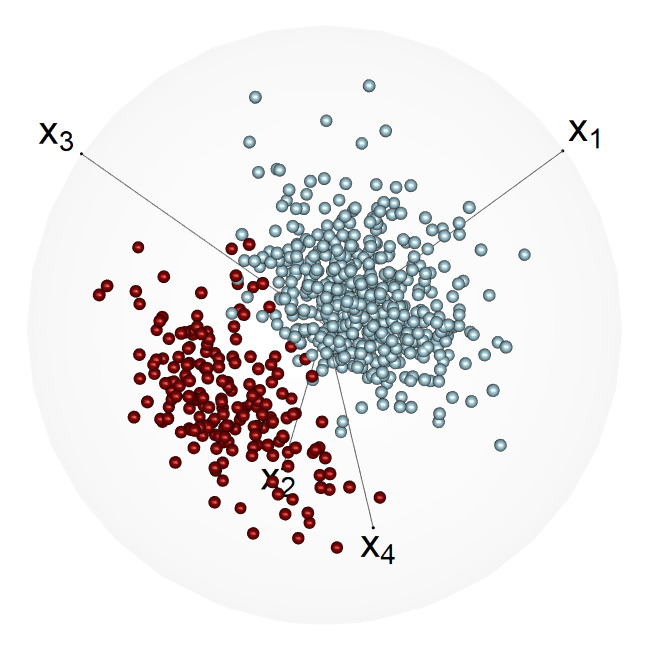}}
  \subfloat[]{\includegraphics[width=0.25\textwidth]{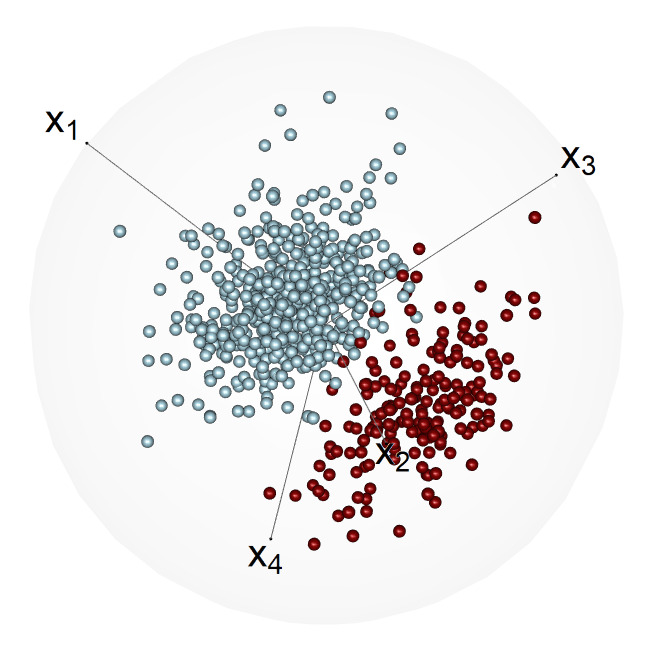}}
}
 \par\bigskip
 \centering
 \text{
 \tikz\draw[black,fill=dkp1] (0,0) circle (.5ex); Normal 
 \tikz\draw[black,fill=dkp2] (0,0) circle (.5ex); ASD subject
 }
 \caption{RadViz3D displays of the ASD screening dataset.\vspace{-1em}}
 \label{autism}
\end{figure}
The two groups are reasonably well-separated, with minimum
classification error of 1.7\%~\citep{rajandmasood20}, and therefore
should be easily separated.  RadViz3D (Fig.~\ref{autism} and Fig.~S7f)
performs better in separating two groups than $t$-SNE (Fig.~S7a),
star coordinates with OLDA (Fig.~S7c) or ULDA (Fig.~S7d), and Viz3D
(Fig.~S7e). UMAP (Fig.~S7b) also separates the two groups, but they
are completely disjoint and does not accurately reflect the best
misclassification rate of 1.7\% for this dataset. 

 
 \subsubsection{SPECT heart dataset}

 \begin{figure}[h]
   \centering
 \mbox{
 \subfloat[]{\includegraphics[width=0.25\textwidth]{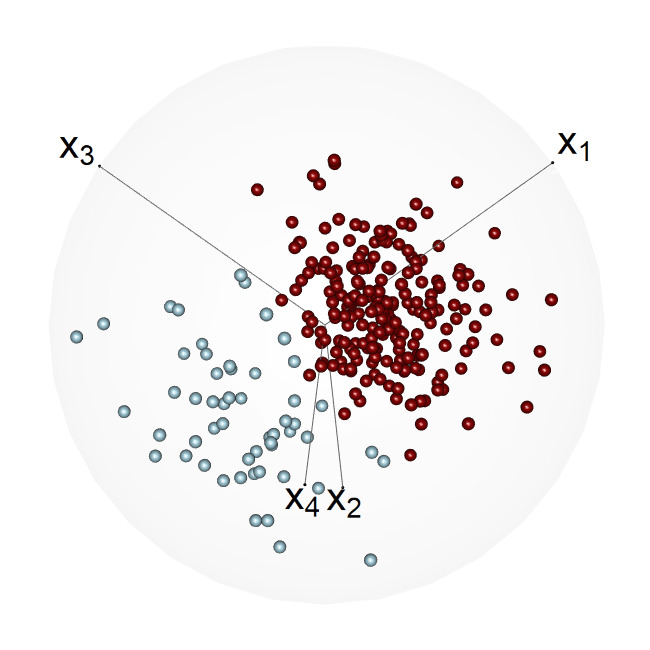}}
 \subfloat[]{\includegraphics[width=0.25\textwidth]{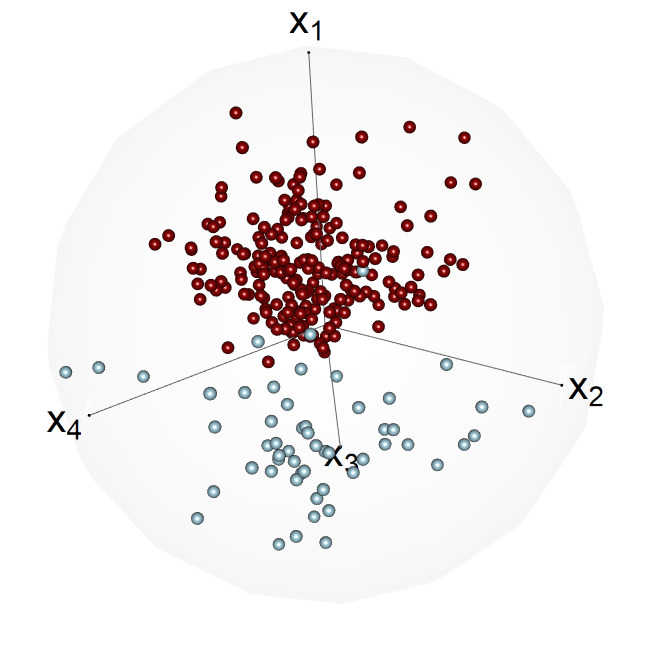}}
 }
 \par\bigskip
 \centering
 \text{
 \tikz\draw[black,fill=dkp1] (0,0) circle (.5ex); Normal
 \tikz\draw[black,fill=dkp2] (0,0) circle (.5ex); Abnormal
 }
 \caption{RadViz3D displays of the SPECT Heart dataset.}
 \label{spect}
   \vspace{-0.1in}
 \end{figure}
  This dataset~\cite{kurganetal01} from 
 the UCIMLR~\citep{newmanetal98}
 has 22 binary 
  attributes that summarize cardiac Single Proton Emission
 Computed Tomography (SPECT) images of 55 normal and 212 abnormal
 patients. 
Each image was summarized by means of 44 continuous features that was
further processed to obtain 22 binary
features~\citep{kurganetal01}. 
The separation between these two groups is very good with current
classification results~\citep{yadavetal20} finding small but positive
misclassification rates. 
In this example, RadViz3D (Fig.~\ref{spect} and Fig.~S8f) performs
well in separating the two groups with very small overlap. UMAP
(Fig.~S7b) again produces disjoint groups in the visualization and
fails to accurately reflect the non-zero misclassification rate
between the two groups while $t$-SNE (Fig.~S8a), Viz3D (Fig.~S8e) and
especially star coordinate plots  with
OLDA (Fig.~S8c) or ULDA (Fig.~S8d) fail to clearly display the
separability of the two groups. 


 
 \subsection{Datasets with mixed features}
 \label{app:mixed}
We now illustrate performance of RadViz3D on two real datasets with
continuous and discrete-valued features. No method can currently display
such datasets, so as in Section~\ref{app:discrete}, we employ the GDT
 (and the MRP for RadViz3D and Viz3D) on the datasets before their display. 
\subsubsection{Indic scripts}
This dataset~\citep{obaidullahetal18} is on 116 different
features from handwritten scripts of 11 Indic languages. We choose a
subset of 5 languages from 4 regions, namely Bangla
(from the east),   Gurmukhi (north), Gujarati (west), Kannada and
Malayalam (languages from the neighboring southern states of Karnataka and
 Kerala) and a sixth language (Urdu, with a distinct Persian
 script). 
\begin{figure}[h]
  \centering
  \mbox{\subfloat{\includegraphics[width=0.5\textwidth]{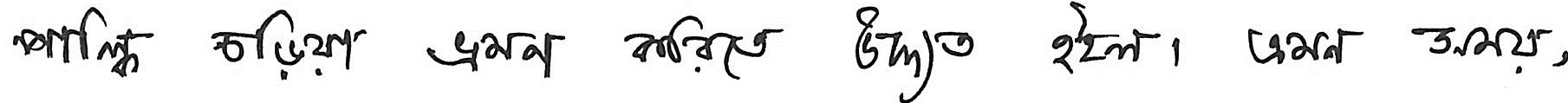}}}
  \mbox{\subfloat{\includegraphics[width=0.5\textwidth]{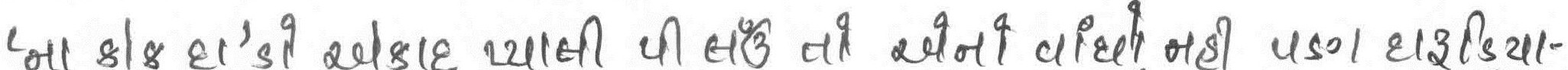}}}
      \mbox{\subfloat{\includegraphics[width=0.5\textwidth]{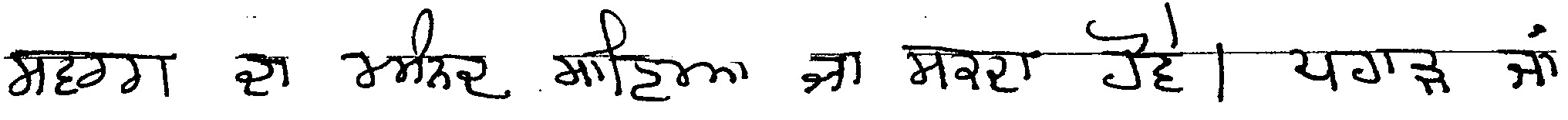}}}
  \mbox{\subfloat{\includegraphics[width=0.5\textwidth]{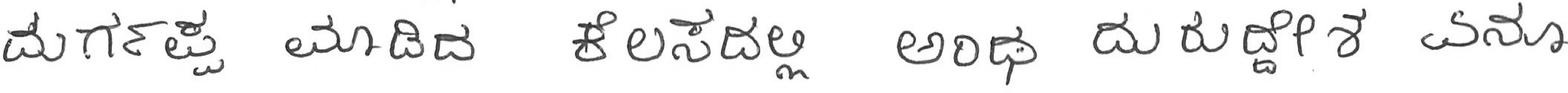}}}
  \mbox{\subfloat{\includegraphics[width=0.5\textwidth]{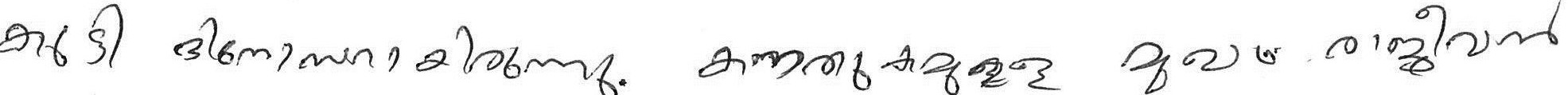}}}
  \setcounter{subfigure}{0}
\vspace{-0.1in}
  \mbox{\subfloat[Handwriting samples from (top to bottom) Bangla,
    Gujarati, Gurmukhi, Kannada, Malayalam and Urdu]{\label{scripts}\includegraphics[width=0.5\textwidth]{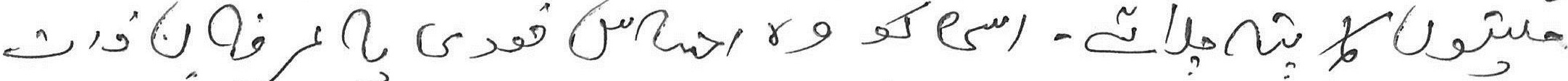}}}
  \mbox{
\subfloat[]{\label{script_cd:radviz3d}\includegraphics[width=0.25\textwidth]{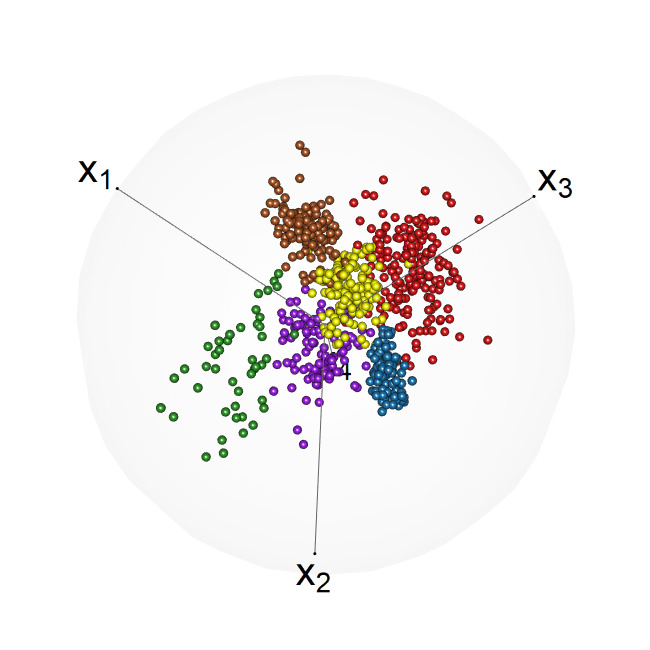}}
\subfloat[]{\label{script_cd:radviz3d-2}\includegraphics[width=0.25\textwidth]{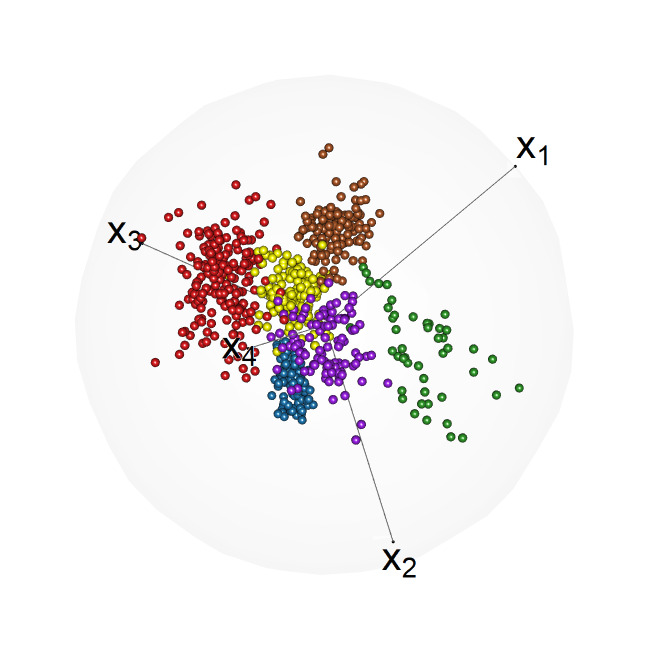}
}}
\par\bigskip
\centering
{\small
  \text{
    \hspace{-1.5ex}    
    \tikz\draw[black,fill=indic1] (0,0) circle (.5ex); Bangla
    \tikz\draw[black,fill=indic2] (0,0) circle (.5ex); Gujarati
    \tikz\draw[black,fill=indic3] (0,0) circle (.5ex); Gurmukhi
    \tikz\draw[black,fill=indic4] (0,0) circle (.5ex); Kannada
    \tikz\draw[black,fill=indic5] (0,0) circle (.5ex); Malayalam
    \tikz\draw[black,fill=indic6] (0,0) circle (.5ex); Urdu
}}
\caption{Indic scripts: (a) Samples and (b,c) RadViz3D displays.}
\label{script_cd}
\end{figure}
Figure.~\ref{scripts} displays a 
line from a sample document in each script and illustrates the
challenges in characterizing handwritten scripts because of the
additional effect of individual handwriting styles. 
The challenges of accounting for handwriting variability, and the 
distinctiveness of the six scripts are captured well in the RadViz3D displays
(Figs.~\ref{script_cd:radviz3d},~\ref{script_cd:radviz3d-2} and S9f). 
For example, we see that Kannada and Malayalam
are close by in the displays. The three Sanskrit-based languages
of Bangla, Gujarati and Gurmukhi are neighbors of each other in the
displays. The placement of Urdu farther from the rest but still close
to both Gujarati and Bangla indicates the possible influence of Persian, for
reasons of history and geography, on the  handwriting of these
scripts. The RadViz3D displays therefore make intuitive sense for this dataset.
The $t$-SNE~(Fig.~S9a) and UMAP~(Fig.~S9b) displays distinguish the
scripts very well but do not depict similarities between
them. Indeed, the UMAP display separates even individuals from
within the scripts, and in our view, performs poorly. Viz3D~(Fig.~S9e)
distinguishes Urdu, Kannada and Gujarati very well but not the
 other languages while star coordinate plots~(Figs.~S9c,d) with OLDA
 and ULDA do poorly.  
\subsubsection{RNA sequences of human tissues}
This dataset \citep{wangetal18} consists of gene expression levels, in FPKM
(Fragments per Kilobase of transcripts per Million), of RNA sequences
from 13 human organs. from which we choose the eight largest (in terms
of available samples) organs, -- esophagus (659
samples), colon (339), thyroid (318), lung (313), breast (212),
stomach (159), liver (115) and prostate (106) -- for our
illustration. This dataset has $p\!\!=\!\!20242$ discrete features,
however some of them have so many discrete values, and the
probabilities for these discrete values are almost equal and very
small. So, these features have approximately continuous CDFs and can
be considered to be continuous. 
Thus, this dataset essentially has both discrete and continuous attributes. 
\begin{figure}[h]
	\vspace{-1.5em}
\mbox{
  \subfloat[]{\includegraphics[width=0.25\textwidth]{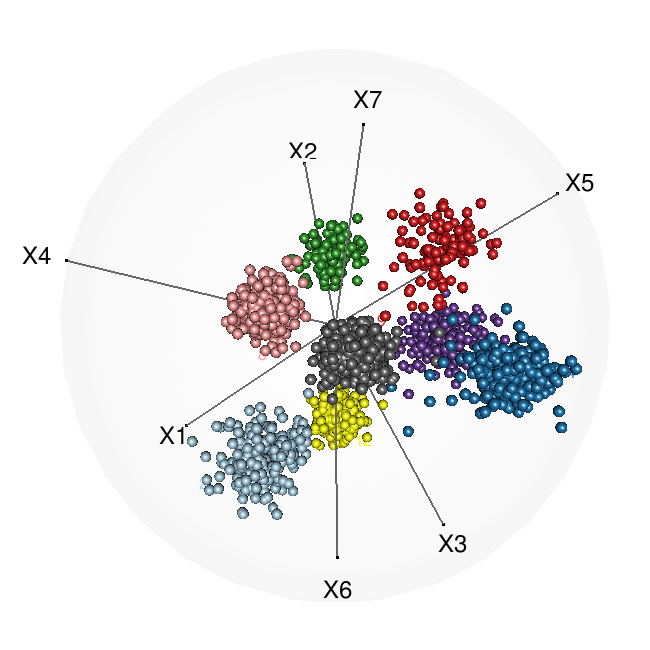}}
  \subfloat[]{\includegraphics[width=0.25\textwidth]{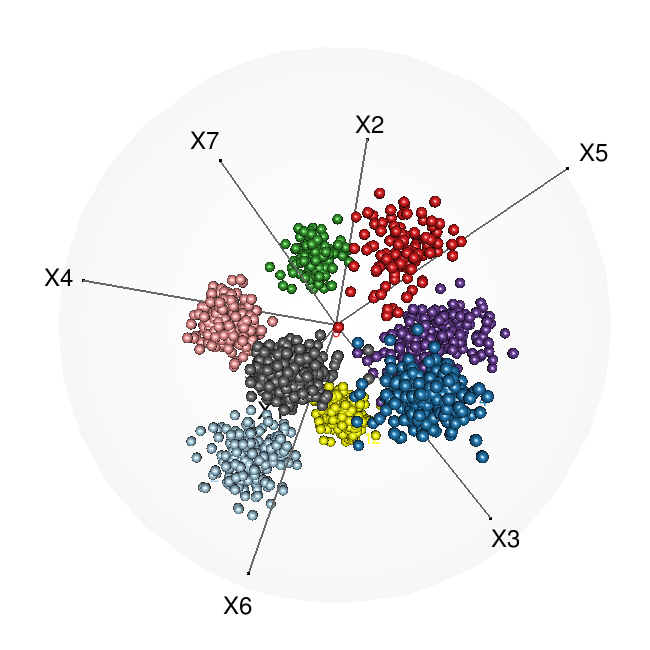}}
}
\par\bigskip
\text{
\tikz\draw[black,fill=gtex1] (0,0) circle (.5ex); Breast
\tikz\draw[black,fill=gtex2] (0,0) circle (.5ex); Colon
\tikz\draw[black,fill=gtex3] (0,0) circle (.5ex); Esophagus
\tikz\draw[black,fill=gtex4] (0,0) circle (.5ex); Liver
\tikz\draw[black,fill=gtex5] (0,0) circle (.5ex); Lung
\tikz\draw[black,fill=gtex6] (0,0) circle (.5ex); Prostate}
\text{
\tikz\draw[black,fill=gtex7] (0,0) circle (.5ex); Stomach
\tikz\draw[black,fill=gtex8] (0,0) circle (.5ex); Thyrioid
}
\caption{RadViz3D displays of the RNA-seq dataset.}
\label{gtex}
\end{figure}
The RadViz3D(Figs.~\ref{gtex} and~S10f) display shows very clear
separation between almost all the organs, except for the colon and the
stomach that have marginal overlap while $t$-SNE~(Fig.~S10a) and
UMAP~(Fig.~S10b) show unclear distinction of  the organs and are
unsatisfactory. Viz3D~(Fig.~S10e) and star coordinate plots~(Figs.~S10c,d) with OLDA
and ULDA also perform worse than RadViz3D.


\section{Discussion}
\label{sec:conclusion}
We develop methods for the display of labeled
high-dimensional observations. We develop the MRP to
summarize a dataset before display, such that the projected data have
uncorrelated coordinates while also separating the different groups. A
third contribution of this paper is to display data with mixed
features so we propose using the GDT to transform features to have
standard normal marginals, while preserving the correlation structure
using copulas.  Labeled datasets with discrete-valued, mixed or heavily-skewed
attributes are transformed, after removing redundant features, to the 
continuous space using the GDT, following which they are displayed using 
the MRP and RadViz3D. Our methodology displays both 
the distinctiveness and complexity of labeled data well. 
The GDT and the MRP are 
general transformation and data reduction methods that we show can
also be used with other visualization techniques.  
An R package {\tt radviz3d} 
implementing our methodology is publicly available at
\url{https://github.com/fanne-stat/radviz3d/}. 

Some aspects of our development could benefit from further
attention. For instance, the MRP is a linear projection method that is 
designed to maximize separation between labeled data. It would be
interesting to see if non-linear projections can improve
results. Also, for unlabeled data, we suggest the use of PCs in place
of the MRP, but other transformations may provide better
visualizations. Also, the GDT is 
inapplicable to datasets with features that have more than two nominal
categories. It would be important to develop 
methodology  for such datasets. Thus, we see that
while this paper has made important 
contributions in visualizing mixed-attribute datasets,  there remain
issues that merit additional investigation.
 \section*{Acknowledgments}
 The authors thank  S. Dutta for helpful
  discussions. A portion of this manuscript won the first author a 2021 Student 
Paper Competition award from the American Statistical Association
(ASA) Section on Statistical Graphics. The third author's  research was
supported in  part by the United States Department
  of Agriculture (USDA)/National Institute of Food and
  Agriculture (NIFA), Hatch project IOW03617. The content of this paper however is
  solely the responsibility of the  authors and does not represent the
  official views of the USDA. 
\ifCLASSOPTIONcaptionsoff
  \newpage
\fi
\bibliographystyle{IEEEtran}
\bibliography{references}

\begin{thebibliography}{10}
\providecommand{\url}[1]{#1}
\csname url@samestyle\endcsname
\providecommand{\newblock}{\relax}
\providecommand{\bibinfo}[2]{#2}
\providecommand{\BIBentrySTDinterwordspacing}{\spaceskip=0pt\relax}
\providecommand{\BIBentryALTinterwordstretchfactor}{4}
\providecommand{\BIBentryALTinterwordspacing}{\spaceskip=\fontdimen2\font plus
\BIBentryALTinterwordstretchfactor\fontdimen3\font minus
  \fontdimen4\font\relax}
\providecommand{\BIBforeignlanguage}[2]{{%
\expandafter\ifx\csname l@#1\endcsname\relax
\typeout{** WARNING: IEEEtran.bst: No hyphenation pattern has been}%
\typeout{** loaded for the language `#1'. Using the pattern for}%
\typeout{** the default language instead.}%
\else
\language=\csname l@#1\endcsname
\fi
#2}}
\providecommand{\BIBdecl}{\relax}
\BIBdecl

\bibitem{cardetal99}
S.~K. Card, J.~D. Mackinlay, and B.~Schneiderman, \emph{Readings in information
  visualization: using vision to think}.\hskip 1em plus 0.5em minus 0.4em\relax
  Morgan Kaufmann, 1999.

\bibitem{bertinietal11}
E.~Bertini, A.~Tatu, and D.~Keim, ``Quality metrics in high-dimensional data
  visualization: an overview and systematization,'' \emph{{IEEE} {T}ransactions
  on {V}isualization and {C}omputer {G}raphics}, vol.~17, no.~12, p.
  2203–2212, 2011.

\bibitem{chambersetal83}
J.~M. Chambers, W.~S. Cleveland, B.~Kleiner, and P.~A. Tukey, \emph{Graphical
  Methods for Data Analysis}.\hskip 1em plus 0.5em minus 0.4em\relax Belmont,
  CA: Wadsworth, 1983.

\bibitem{chernoff73}
H.~Chernoff, ``The use of faces to represent points in k-dimensional space
  graphically,'' \emph{Journal of the American Statistical Association},
  vol.~68, no. 342, pp. 361--368, 1973.

\bibitem{inselberg85}
A.~Inselberg, ``The plane with parallel coordinates,'' \emph{The Visual
  Computer}, vol.~1, pp. 69--91, 1985.

\bibitem{wegman90}
E.~Wegman, ``Hyperdimensional data analysis using parallel coordinates,''
  \emph{Journal of the American Statistical Association}, vol.~85, pp.
  664--675, 1990.

\bibitem{fayyadetal01}
U.~Fayyad, G.~Grinstein, and A.~Wierse, \emph{Information Visualization in Data
  Mining and Knowledge Discovery}.\hskip 1em plus 0.5em minus 0.4em\relax
  Morgan Kaufmann, 2001.

\bibitem{andrews72}
D.~F. Andrews, ``Plots of high-dimensional data,'' \emph{Biometrics}, vol.~28,
  no.~1, pp. 125--136, 1972.

\bibitem{khattreeandnaik02}
R.~Khattree and D.~N. Naik, ``Andrews plots for multivariate data: Some new
  suggestions and applications,'' \emph{Journal of Statistical Planning and
  Inference}, vol. 100, no.~2, pp. 411--425, 2002.

\bibitem{gabriel71}
K.~R. Gabriel, ``The biplot graphical display of matrices with application to
  principal component analysis,'' \emph{Biometrika}, vol.~58, pp. 453--467,
  1971.

\bibitem{kandogan01}
\BIBentryALTinterwordspacing
E.~Kandogan, ``Visualizing multi-dimensional clusters, trends, and outliers
  using star coordinates,'' in \emph{Proceedings of the Seventh ACM SIGKDD
  International Conference on Knowledge Discovery and Data Mining}, ser. KDD
  '01.\hskip 1em plus 0.5em minus 0.4em\relax New York, NY, USA: ACM, 2001, pp.
  107--116. [Online]. Available: \url{http://doi.acm.org/10.1145/502512.502530}
\BIBentrySTDinterwordspacing

\bibitem{mcinnesetal18}
L.~McInnes, J.~Healy, N.~Saul, and L.~Grossberger, ``Umap: Uniform manifold
  approximation and projection,'' \emph{Journal of Open Source Software},
  vol.~3, p. 861, 09 2018.

\bibitem{hoffmanetal97}
P.~Hoffman, G.~Grinstein, K.~Marx, I.~Grosse, and E.~Stanley, ``{DNA} visual
  and analytic data mining,'' in \emph{Proceedings of the 8th conference on
  Visualization '97, {VIS’97}}.\hskip 1em plus 0.5em minus 0.4em\relax IEEE
  Computer Society Press, 1997, p. 437–441.

\bibitem{hoffmanetal99}
P.~Hoffman, G.~Grinstein, and D.~Pinkney, ``Dimensional anchors: a graphic
  primitive for multidimensional multivariate information visualizations,'' in
  \emph{Proceedings of the 1999 workshop on new paradigms in information
  visualization and manipulation in conjunction with the eighth ACM internation
  conference on Information and knowledge management}.\hskip 1em plus 0.5em
  minus 0.4em\relax ACM, 1999, pp. 9--16.

\bibitem{grinsteinetal01}
G.~G. Grinstein, C.~B. Jessee, P.~E. Hoffman, P.~J. O’Neil, and A.~G. Gee,
  ``High-dimensional visualization support for data mining gene expression
  data,'' in \emph{{DNA} Arrays: Technologies and Experimental Strategies},
  E.~V. Grigorenko, Ed.\hskip 1em plus 0.5em minus 0.4em\relax Boca Raton,
  Florida: CRC Press LLC, 2001, ch.~6, pp. 86--131.

\bibitem{draperetal09}
G.~M. {Draper}, Y.~{Livnat}, and R.~F. {Riesenfeld}, ``A survey of radial
  methods for information visualization,'' \emph{IEEE Transactions on
  Visualization and Computer Graphics}, vol.~15, no.~5, pp. 759--776, Sep.
  2009.

\bibitem{arteroanddeoliveira04}
A.~O. Artero and M.~C.~F. de~Oliveira, ``Viz3d: effective exploratory
  visualization of large multidimensional data sets,'' in \emph{Proceedings.
  17th Brazilian Symposium on Computer Graphics and Image Processing}, Oct
  2004, pp. 340--347.

\bibitem{zhu2021fully}
Y.~Zhu, F.~Dai, and R.~Maitra, ``Fully three-dimensional radial
  visualization,'' 2021.

\bibitem{ruschendorf13}
L.~R{\"u}schendorf, \emph{Mathematical Risk Analysis}.\hskip 1em plus 0.5em
  minus 0.4em\relax Berlin Heidelberg: Springer-Verlag, 2013.

\bibitem{benjaminiandhochberg95}
Y.~Benjamini and Y.~Hochberg, ``Controlling the false discovery rate: a
  practical and powerful approach to multiple testing,'' \emph{Journal of the
  Royal Statistical Society}, vol.~57, pp. 289--300, 1995.

\bibitem{korenandcarmel04}
Y.~Koren and L.~Carmel, ``Robust linear dimensionality reduction,'' \emph{IEEE
  Transactions on Visualization and Computer Graphics}, vol.~10, no.~4, pp.
  459--470, July 2004.

\bibitem{golub1996}
\BIBentryALTinterwordspacing
G.~Golub, C.~Van~Loan, C.~Van~Loan, and P.~Van~Loan, \emph{Matrix
  Computations}, ser. Johns Hopkins Studies in the Mathematical Sciences.\hskip
  1em plus 0.5em minus 0.4em\relax Johns Hopkins University Press, 1996.
  [Online]. Available: \url{https://books.google.com/books?id=mlOa7wPX6OYC}
\BIBentrySTDinterwordspacing

\bibitem{ye05}
J.~Ye, ``Characterization of a family of algorithms for generalized
  discriminant analysis on undersampled problems.'' \emph{Journal of Machine
  Learning Research}, vol.~6, no.~4, 2005.

\bibitem{jinetal01}
Z.~Jin, J.-Y. Yang, Z.-S. Hu, and Z.~Lou, ``Face recognition based on the
  uncorrelated discriminant transformation,'' \emph{Pattern recognition},
  vol.~34, no.~7, pp. 1405--1416, 2001.

\bibitem{melnykovetal12}
\BIBentryALTinterwordspacing
V.~Melnykov, W.-C. Chen, and R.~Maitra, ``{MixSim}: An {R} package for
  simulating data to study performance of clustering algorithms,''
  \emph{Journal of Statistical Software}, vol.~51, no.~12, pp. 1--25, 2012.
  [Online]. Available: \url{http://www.jstatsoft.org/v51/i12/}
\BIBentrySTDinterwordspacing

\bibitem{R}
\BIBentryALTinterwordspacing
{R Development Core Team}, ``R: A language and environment for statistical
  computing,'' R Foundation for Statistical Computing, Vienna, Austria, 2018,
  {ISBN} 3-900051-07-0. [Online]. Available: \url{http://www.R-project.org}
\BIBentrySTDinterwordspacing

\bibitem{maitraandmelnykov10}
R.~Maitra and V.~Melnykov, ``Simulating data to study performance of finite
  mixture modeling and clustering algorithms,'' \emph{Journal of Computational
  and Graphical Statistics}, vol.~19, no.~2, pp. 354--376, 2010.

\bibitem{melnykovandmaitra11}
V.~Melnykov and R.~Maitra, ``{CARP}: Software for fishing out good clustering
  algorithms,'' \emph{Journal of Machine Learning Research}, vol.~12, pp. 69 --
  73, 2011.

\bibitem{banerjeeetal08}
O.~Banerjee, L.~E. Ghaoui, and A.~d’Aspremont, ``Model selection through
  sparse maximum likelihood estimation for multivariate gaussian or binary
  data,'' \emph{Journal of Machine Learning Research}, vol.~9, pp. 485--516,
  2008.

\bibitem{thabtah17}
F.~Thabtah, ``Autism spectrum disorder screening: machine learning adaptation
  and {DSM-5} fulfillment,'' in \emph{Proceedings of the 1st International
  Conference on Medical and Health Informatics 2017}.\hskip 1em plus 0.5em
  minus 0.4em\relax ACM, 2017, pp. 1--6.

\bibitem{newmanetal98}
\BIBentryALTinterwordspacing
D.~J. Newman, S.~Hettich, C.~L. Blake, and C.~J. Merz, ``{UCI} repository of
  machine learning databases,'' 1998. [Online]. Available:
  \url{http://www.ics.uci.edu/$\sim$mlearn/MLRepository.html}
\BIBentrySTDinterwordspacing

\bibitem{rajandmasood20}
S.~Raj and S.~Masood, ``Analysis and detection of autism spectrum disorder
  using machine learning techniques,'' \emph{Procedia Computer Science}, vol.
  167, pp. 994--1004, 2020, international Conference on Computational
  Intelligence and Data Science.

\bibitem{kurganetal01}
L.~A. Kurgan, K.~J. Cios, R.~Tadeusiewicz, M.~R. Ogiela, and L.~S. Goodenday,
  ``Knowledge discovery approach to automated cardiac {SPECT} diagnosis,''
  \emph{Artificial Intelligence in Medicine}, vol.~23, no.~2, pp. 149--169,
  2001.

\bibitem{yadavetal20}
S.~S. {Yadav}, S.~M. {Jadhav}, R.~G. {Bonde}, and S.~T. {Chaudhari},
  ``Automated cardiac disease diagnosis using support vector machine,'' in
  \emph{2020 3rd International Conference on Communication System, Computing
  and IT Applications (CSCITA)}, 2020, pp. 56--61.

\bibitem{obaidullahetal18}
\BIBentryALTinterwordspacing
S.~M. Obaidullah, C.~Halder, K.~C. Santosh, N.~Das, and K.~Roy,
  ``Phdindic{\_}11: page-level handwritten document image dataset of 11
  official indic scripts for script identification,'' \emph{Multimedia Tools
  and Applications}, vol.~77, no.~2, pp. 1643--1678, Jan 2018. [Online].
  Available: \url{https://doi.org/10.1007/s11042-017-4373-y}
\BIBentrySTDinterwordspacing

\bibitem{wangetal18}
Q.~Wang, J.~Armenia, C.~Zhang, A.~Penson, E.~Reznik, L.~Zhang, T.~Minet,
  A.~Ochoa, B.~Gross, C.~A.~Iacobuzio-Donahue, D.~Betel, B.~S.~Taylor, J.~Gao,
  and N.~Schultz, ``Unifying cancer and normal {RNA} sequencing data from
  different sources,'' \emph{Scientific Data}, vol.~5, p. 180061, 04 2018.

\end{thebibliography}



%




\end{document}


\setcounter{page}{7}
\title{\large Appendix to ``Three-dimensional Radial Visualization of
  High-dimensional Datasets with Mixed Features'' by
  Yifan Zhu, Fan Dai and Ranjan Maitra}




\renewcommand\thefigure{A-\arabic{figure}}


\maketitle
\subsection*{Summary of other visualization methods mentioned in the Introduction}
Here, as suggested by a reviewer, we briefly describe the other common visualization methods, using the gamma ray bursts dataset of Section 4.1.1 for
illustration. 
\subsubsection*{\underline{Starplots}}
Starplots~\citep{chambersetal83}, also called {\em radarplots} or {\em spiderplots}, are
meant to display individual observations. The base starplot is
suitable for non-negative measurement, so all variables in the data
are standardized and shifted so that the minimum values are set at
zero. We construct one star for each observation in the data with $p$
variables. So, a circle of fixed radius with $p$ equally spaced rays
representing $p$ variables is created, and the value of each variable
(after the standardization and shifting) is represented by the length
of the corresponding ray. Connecting the ends of these rays will give
a star-like shape. Each star corresponds to an observation so this
plot can only be used for a handful of observations at a time. In these
cases, starplots are used to find observations with similar
features. However, even with more than a few observations, 
starplots become impractical to either apply or interpret. We
illustrate the starplot on the GRB data. Because the dataset has 1599
complete observations, we are unable to display it using
starplots. So, instead we display the five group means using starplots
in Figure~\ref{fig:grbstar}.
\begin{figure}[!h]
    \vspace{-3in}
    \centering
    \input{appendix/stars}
    \vspace{-1in}
    \caption{Starplots for the group means of the GRB dataset. The key
      indicating the variable denoted in each ray is provided in the second row.} 
    \label{fig:grbstar}
\end{figure}

\subsubsection*{\underline{Chernoff faces}}
Chernoff faces~\citep{chernoff73} have the same general idea as starplots in the sense
that it uses a face to represent a single observation. Instead of rays
in the starplot, Chernoff faces  uses facial characteristics (length
of nose, position of mouth, etc.) to represent the different
features. We can also distinguish and compare individual observations using
faces. However, similar to starplots, Chernoff faces is not
\begin{figure}[h]
\centering
    \vspace{-0.75in}
    \input{appendix/chernoff}
    \vspace{-4.75in}
    \caption{Chernoff faces of the group means for the GRB data. The
      following facial characteristics correspond to the variables:
      (1) height of face:  $T_{50}$, (2) width of face: $T_{90}$,
      (3) structure of face: $F_1$, (4) height of mouth: $F_2$, (5)
      width of mouth: $F_3$, (6) smiling: $F_4$, (7) height of eyes:
      $P_{64}$, (8) width of eyes: $P_{256}$, (9) height of hair:
      $P_{1024}$, (10) width of hair: $T_{50}$, (11) style of hair:
      $T_{90}$, (12) height of nose: $F_1$, (13) width of nose: $F_2$,
      (14) width of ear: $F_3$, (15) height of ear: $F_4$.}
    \label{fig:grbface}
\end{figure}
appropriate for more than a few (say 10) observations, and therefore
is impractical for datasets such as the GRB. We illustrate Chernoff
faces on the group means of the GRB dataset in Figure \ref{fig:grbface}.
Chernoff faces. 
\subsubsection*{\underline{Parallel coordinate plot}}
The parallel coordinate plot~\citep{inselberg85,wegman90} represents
multidimensional data using a polyline for each observation. The data
are scaled so that the range 
of each dimension is from 0 to 1. For a dataset with $p$ variables and
$n$ observations, $p$ vertical axes are placed in parallel at equal
distances. For one observation, we connect $p-1$ lines the between the
$p$ vertical axes, with the ends of lines representing the scaled
value a variable. This leads a polyline for each observation. Similar
observations will show a similar pattern in the polylines in the
visualization. The order in which the variables are displayed affects
the parallel coordinate plot. Further, it is difficult to display data
with many features, and is also not easy to distinguish patterns with
many observations, as seen from Figure \ref{fig:grbparallel}
which displays the GRB data in a parallel coordinate plot.
\begin{figure}[h]
\centering
    \includegraphics[width=\textwidth]{appendix/grb_parallel.png}
    \caption{Parallel coordinate plot of the group means for the GRB data}
    \label{fig:grbparallel}
\end{figure}

\subsubsection*{\underline{Surveyplot}}
A surveyplot~\citep{fayyadetal01} is a simple technique of extending a
line graph (like a bar plot, each observation representing a bar) to
multiple side-by-side line graphs. For a data with $n$ observations
and $p$ variables, a line graph is created for each variable, and $n$
lines were placed parallelly with the lengths representing the value
of the variable and the positions determined by the observation's
index in the dataset. In the end, we get $p$ side-by-side line
graphs. If we sort the data according to a particular variable and
look at the classes, and then cycle through the variables, we can find
the variable that is most associated with the labels. We can also find
the associations between the values in the different variables by
comparing the values of the ordered feature with that of the other (unordered)
features.  For instance,  Figure~\ref{fig:grbsurvey} displays the
surveyplot, according to the ordering of $F_1$ which has the highest
association with the class labels. We see some relationship between
$F_1$ and $F_2$ and to a lesser extent $F_3$. Beyond these
observations, the value of a surveyplot appears limited. 
\begin{figure}[h]
\centering
\input{appendix/survey}
    \caption{Survey plot of the group means for the GRB data}
    \label{fig:grbsurvey}
\end{figure}

\subsubsection*{\underline{Andrew's Curves}}
Andrew's curves~\citep{andrews72,khattreeandnaik02} display each $p$-dimensional observation $\bx = (x_1, x_2, \ldots, x_p)'$ as a curve using the function
\[f(t) = x_1 + x_2 \sin t + x_3 \cos t + x_4 \sin 2t + x_5 \cos 2t + \ldots\].
The function is usually plotted in the interval $-\pi < t >
\pi$. There are four types of such plots~\citep{khattreeandnaik02}
with the exact type often obtained by trying out the different
displays. An advantage if this method is that it can represent many
dimensions. However, it takes long computing times to do the
calculations and display for high-dimensional datasets. Also,  there
is no interpretability in the figures. Further, class separability is
determined by considering the curves in their entirety and this can be
cumbersome to visualize, as seen in Figure \ref{fig:grbandrews} which
displays Andrews' curves of type 2 for GRB dataset.
\begin{figure}[h]
  \centering
  \includegraphics[width=\textwidth]{appendix/andrews-crop}
    \caption{Andrews' curve (Type 2) for the GRB dataset.}
    \label{fig:grbandrews}
\end{figure}

\subsubsection*{\underline{Biplot}} A biplot~\citep{gabriel71} is
constructed from a singular value decomposition (SVD) of the centered
data matrix $\bX$ to obtain its low-rank approximation. Suppose that
we have a centered $n \times p$ data matrix $X$ (the means of $p$
variables are 0), we first obtain its SVD decomposition:
\[\bX = \sum_{k=1}^p d_k \bu_k \bv_k'.\]

Then two scatterplots are created with the same set of axes. The first is for rows (observations), and each point is:
\[(d_{1}^\alpha u_{1i}, d_{2}^{\alpha}u_{2i}), \, i = 1, 2, \ldots, n.\]

The second is for columns (variables), and each point is:
\[(d_{1}^{1 - \alpha} v_{1j}, d_2^{1 - \alpha} v_{2j}),\, j = 1, 2, \ldots, p.\]
Usually, we use $\alpha = 0$ or $1$. A biplot are essentially displays the first
two principal components of the centered data and as such is an
unsupervised method.  The biplot is suitable if high-dimensional data
can be represented well by its first two principal components, and in
the case of labeled data, if the major proportion of the total
variance in the data is driven by the group differences. Figure
\ref{fig:grbbiplot} illustrates the biplot of GRB data and shows
unclear separation between the groups. It shows that two PCs are
likely not adequate to represent the differences in the data. 
\begin{figure}[h]
\centering
\includegraphics[width=\textwidth]{appendix/biplot-crop}
\caption{Biplot of the GRB dataset.}
    \label{fig:grbbiplot}
\end{figure}

\subsubsection*{\underline{Star Coordinates plot}}
The fundamental idea governing a star coordinates
plot~\citep{kandogan01} is to arrange the $p$ axes (for the $p$
features) on a 2D plane, where the coordinate axes are not necessary
orthogonal to each other. The data are transformed with a min-max
transformation so that all variables have a range from 0 to 1. Then
the $p$ variables in each observation are converted to 2D unit vectors (usually equi-spaced on the unit circle) and the linear combination of these unit vectors is used to represent observations. 
Let $\bx_1, \bx_2, \ldots, \bx_n$ be $p$-dimensional observations after the min-max transformation, and $\bu_1, \bu_2, \ldots, \bu_p$ be the unit vectors for $p$ variables, we represent $\bx_i$ by
\[\bp_i = \sum_{j=1}^p x_{ij} \bu_j.\]
Usually, we use $\bu_j = (\cos (2\pi (j-1)/p), \sin (2\pi (j-1)/p)).$

Star coordinate plots struggle with high-dimensional data since the
$p$ transformed axes on the 2D plane get harder to separate with
larger $p$. The results can also lack interpretability as seen in
Figure \ref{fig:grbstarcoord} that shows the GRB data by means of a
star coordinate plot. We see that the groups are not very easily
distinguished. 
\begin{figure}[h]
\centering
\input{appendix/starcoord}
\caption{Star coordinate plot of the GRB dataset.}
    \label{fig:grbstarcoord}
\end{figure}

\subsubsection*{\underline{Uniform Manifold Approximations and
    Projections (UMAP)}}
UMAP~\citep{mcinnesetal18} is a nonlinear dimension reduction technique
that can be used to visualize high-dimensional data. The data are
assumed to be uniformly distributed on a Riemannian manifold that can
be modeled with a fuzzy topological structure. Then UMAP finds a
lower-dimensional representation of the data that has the closest
equivalent fuzzy topological structure.
For visualization, it makes sense to choose two or three
projections. However, it 
is hard to compare the similarity between different groups of data by
UMAP. We note that UMAP is really a
classification tool and is geared towards finding the best
classification rule. As such it is unable to correctly characterize
the difficulty of separating out classes and making distinctions given
its sole focus on classification.  For example, the GRB dataset is
illustrated by means of a 3D UMAP in Figure~S2b. We see that the five
groups are very well-separated with no hint of the controversy between
2, 3 or 5 groups as described in Section 4.1.1 of the paper. 

\ifCLASSOPTIONcaptionsoff
  \newpage
\fi
\bibliographystyle{IEEEtran}
\bibliography{references}